\documentclass[11pt]{article}

\usepackage{microtype}
\usepackage{graphicx}
\usepackage{subfigure}
\usepackage{booktabs} 

\usepackage{hyperref}

\usepackage{mathrsfs}
\usepackage{amsmath}
\usepackage[english]{babel}
\usepackage[toc,page]{appendix} 
\usepackage{hyperref}
\usepackage{graphicx,wrapfig,lipsum,xcolor}	
\usepackage{yfonts}	 		
\usepackage{amssymb}
\usepackage{amsmath}
\usepackage{amsthm}
\usepackage{scalerel}
\usepackage{verbatim}

\usepackage{thmtools} 
\usepackage{thm-restate}

\usepackage{dsfont}
\usepackage{amsmath}
\usepackage{mathrsfs}
\usepackage{amssymb, graphicx, amsmath, amsthm}

\def \R{\mathbb{R}}

\def \calA{\mathcal{A}}

\def \calN{\mathcal{N}}

\def \E{\mathbb{E}}

\DeclareMathOperator*{\tr}{tr}

\usepackage{tikz} 
\newtheorem{definition}{Definition}
\newtheorem{example}{Example}

\newtheorem{claim}{Claim}
\newtheorem{lemma}{Lemma}

\usepackage{amsthm}

\allowdisplaybreaks

\usepackage{amsmath}
\usepackage{amssymb}
\usepackage{mathtools}
\usepackage{amsthm}
\usepackage{enumitem}  

\usepackage{array}

\usepackage[capitalize,noabbrev]{cleveref}

\usepackage[textsize=tiny]{todonotes}

\usepackage{algorithmic}
\usepackage{algorithm}

\usepackage[margin = 1in]{geometry}
\usepackage{palatino}
\usepackage{xcolor}
\usepackage[round]{natbib}

\definecolor{darkblue}{rgb}{0.0,0.0,0.65}
\definecolor{darkred}{rgb}{0.68,0.05,0.0}
\definecolor{darkgreen}{rgb}{0.0,0.29,0.29}
\definecolor{darkpurple}{rgb}{0.47,0.09,0.29}
\usepackage{hyperref}
\hypersetup{
   colorlinks = true,
   citecolor  = darkblue,
   linkcolor  = darkred,
   filecolor  = darkblue,
   urlcolor   = darkblue,
 }
\usepackage{url}

\usepackage[utf8]{inputenc} 
\usepackage[T1]{fontenc}    
\usepackage{hyperref}       
\usepackage{url}            
\usepackage{booktabs}       
\usepackage{amsfonts}       
\usepackage{nicefrac}       
\usepackage{microtype}      
\usepackage{xcolor}         
\usepackage{bm}
\usepackage{pgfplots}
\usepackage{cleveref}
\usepackage{tikz}
\usetikzlibrary{positioning, decorations.pathreplacing}

\usepackage{amsmath}
\usepackage{amssymb}
\usepackage{amsthm}

\usepackage[flushleft]{threeparttable}

\usepackage{mathtools}

\title{A Universal Class of Sharpness-Aware Minimization Algorithms}

\author{Behrooz Tahmasebi\\MIT CSAIL\\ \texttt{bzt@mit.edu} \and Ashkan Soleymani\\MIT LIDS\\ \texttt{ashkanso@mit.edu} \and Dara Bahri\\ Google DeepMind \\ \texttt{dbahri@google.com} \and Stefanie Jegelka\\ TU Munich and MIT CSAIL\\ \texttt{stefje@mit.edu} \and Patrick Jaillet\\MIT LIDS\\ \texttt{jaillet@mit.edu}}
\date{}

\begin{document}

\maketitle

\begin{abstract}
Recently, there has been a surge in interest in developing optimization algorithms for overparameterized models as achieving generalization is believed to require algorithms with suitable biases. This interest centers on minimizing sharpness  of  the original loss function; the Sharpness-Aware Minimization (SAM) algorithm has proven effective.  However, most literature only considers a few sharpness measures, such as the maximum eigenvalue or trace of the training loss Hessian, which may not yield meaningful insights for non-convex optimization scenarios like neural networks. Additionally, many sharpness measures are sensitive to parameter invariances in neural networks, magnifying significantly under rescaling parameters.
Motivated by these challenges, we introduce a new class of sharpness measures in this paper, leading to new sharpness-aware objective functions. We prove that these measures are \textit{universally expressive}, allowing any function of the training loss Hessian matrix to be represented by appropriate hyperparameters. Furthermore, we show that the proposed objective functions explicitly bias towards minimizing their corresponding sharpness measures, and how they allow meaningful applications to 
models with parameter invariances (such as scale-invariances). Finally, as instances of our proposed general framework, we present \textit{Frob-SAM} and \textit{Det-SAM}, which are  specifically designed to minimize the Frobenius norm and the determinant of the Hessian of the training loss, respectively. We also demonstrate the advantages of our general framework through extensive  experiments.
\end{abstract}

\section{Introduction}

Understanding the generalization capabilities of overparameterized networks is a fundamental, yet unsolved challenge in deep learning. It is postulated that achieving near-zero training loss alone may be insufficient as there exist many instances where global minima fail to exhibit satisfactory generalization performance. To this end, a dominant observation asserts that the characteristics of the loss landscape play a pivotal role in determining which parameters have low training loss while also exhibiting generalization capabilities.

A recently proposed approach to consider the geometric aspects of the loss landscape, with the aim of achieving generalization, entails the avoidance of sharp minima. For example, the celebrated Sharpness-Aware Minimization (SAM) algorithm has shown enhancements in generalization across many practical tasks \citep{foret2020sharpness}. While the concept of sharpness lacks a precise definition in a general sense, people often introduce various measures to quantify it in practice~\citep{dinh2017sharp}. Many sharpness measures in the literature rely on the second-order derivative characteristics of the training loss function, such as the trace or the operator norm of the Hessian matrix~\citep{chaudhari2019entropy,keskar2016large}.

Nevertheless, traditional methodologies for quantifying sharpness may not suffice to ensure generalization, given the intricate geometry of the loss landscape, which may necessitate different regularization techniques. Moreover, many existing sharpness measures fail to encapsulate the genuine essence of sharpness in deep neural networks because the Hessian matrix no longer maintains positive semi-definiteness.
Furthermore, neural networks exhibit parameter invariances, wherein different parameterizations can yield identical functions — such as scaling invariances in ReLU networks. Consequently, an effective measure of sharpness should remain invariant in the face of such parameter variations. Unfortunately, conventional approaches for quantifying sharpness frequently fall short in addressing this phenomenon.

Therefore, a fundamental question arises: how can one succinctly represent all measures of sharpness within a compact parameterized framework that also enables meaningful applications to models with parameter invariances? This question holds significance in applications as it allows \textit{learning/designing the regularization} in cases where information about the geometry of the loss landscape or parameter invariances is provided, either empirically or through assumption. To the best of our knowledge, this question has remained fairly unexplored in the deep learning literature. 

In this paper, we characterize \textit{all} sharpness measures (i.e., functions of the Hessian of the training loss) through an average-based parameterized representation.  We prove that by changing the (hyper)parameters, the provided representation spans all the sharpness measures as a function of the Hessian matrix. In other words, it is provably a \textit{universal representation}. We also provide quantitative theoretical results on the complexity of the sharpness representation as a function of the data dimension.

Moreover, attached to any representation of sharpness, we provide a new loss function, and we prove that the new loss function is \textit{biased} toward minimizing its corresponding sharpness measure. Since the parameterized representation reduces to SAM (i.e., worst-direction) and average-direction sharpness measures \citep{wen2022does} in special cases, it can be considered as a generalized (hyper)parameterized sharpness-aware minimization algorithm. This generalizes the recent study of the explicit bias of a few sharpness-aware minimization algorithms \citep{wen2022does} to a comprehensive class of objectives. Furthermore, this allows us to readily design algorithms with \textit{any} bias of interest, while to the best of our knowledge, only algorithms with biases towards minimizing the trace, operator norm of the Hessian matrix, and a few other sharpness measures are known in the literature.  As instances of our proposed general algorithm, we present \textit{Frob-SAM} and \textit{Det-SAM}, two new sharpness-aware minimization algorithms that are  specifically designed to minimize the Frobenius norm and the determinant of the Hessian of the training loss function, respectively. 

An interesting feature of the given representation is that it provides a systematic way to construct sharpness measures respecting  parameter invariances, e.g., scale-invariance in neural networks. As a specific example, we provide a class of loss functions and algorithms that are invariant under parameter scaling. Note that the classical sharpness measures, such as trace or the operator norm of the Hessian matrix, are not invariant to rescaling or group actions. 

In the experiments, we explore extensively two specific choices of these algorithms: (1) \textit{Frob-SAM}: an algorithm biased toward minimizing the Frobenius norm of the Hessian, a meaningful sharpness notion for non-convex optimization problems and (2) \textit{Det-SAM}: an algorithm  biased toward minimizing the determinant of the Hessian, a scale-invariant sharpness measure. We demonstrate the advantages of these two cases through an extensive series of experiments.



In short, in this paper we make the following contributions:
\begin{itemize}[leftmargin=10mm,label={$\bullet$},itemsep=0.2ex,topsep=0ex]

\item We propose a new class of sharpness measures, as function of the training loss Hessian. We prove that the new representation is \textit{universally expressive}, meaning that it covers all sharpness measures of the Hessian as its (hyper)parameters change.

\item Along with each sharpness measure we provide an optimization objective and prove that the new objective is explicitly \textit{biased} toward minimizing the corresponding sharpness measure. 

\item The structure of the proposed  method allows meaningful applications to models with parameter invariances, as it provides a class of objective functions for any desired type of parameter invariance.

\item We introduce two fundamental illustrative examples of our proposed general representation and the corresponding algorithms: \textit{Frob-SAM} and \textit{Det-SAM}. \textit{Frob-SAM} is geared towards minimizing the Frobenius norm of the Hessian matrix, providing a meaningful and natural solution to the definition problem of sharpness for non-convex optimization problems. Conversely, \textit{Det-SAM} is focused on minimizing the determinant\footnote{To be more precise, the product of non-zero eigenvalues.} of  Hessian, addressing scale-invariant issues related to parameterization.
\end{itemize}

\section{Related Work}

\citet{foret2020sharpness} recently proposed the Sharpness-Aware Minimization (SAM) algorithm to avoid sharp minima.  The SAM objective has connections to a similar robust optimization problem that was suggested for the study of adversarial attacks in deep learning \citep{madry2017towards}. Besides SAM, \citet{nitanda2023parameter} show how parameter averaging for SGD is biased toward flatter minima.  Label noise SGD also prefers flat minima \citep{damian2021label}. 
\citet{woodworth2020kernel} studied the role of sharpness in overparameterization from a kernel perspective.  See \citep{wang2023sharpness} for the applications of flat minima for domain generalization  (see also \citep{cha2021swad}). For applications of SAM in large language models, see \citep{bahri2022sharpness} (also \citep{zhong2022improving}, and \citep{qu2022generalized, shi2023make} for federated learning). Besides those applications, \citet{wen2023sharpness} prove that current sharpness minimization algorithms sometimes fail to generalize for non-generalizing flattest models.

The  (implicit) bias of many optimization algorithms and architectures has been studied, from the Gradient Descent  (GD)   \citep{ji2019implicit, soudry2018implicit} to mirror descent   \citep{gunasekar2018characterizing, azizan2018stochastic}; see also \citep{gunasekar2018implicit} for linear convolutional networks,  and  \citep{lawrence2021implicit} for equivariant networks. \citet{ji2018gradient} observed that linear neural networks are biased toward weight alignment for different layers (see \citep{le2022training} for non-linear networks). \citet{andriushchenko2022towards} study implicit bias of SAM for diagonal linear networks, and \citet{wen2022does} find the explicit bias of the Gaussian averaging method and other SAM variants.  
 
The role of scale-invariance in generalization in deep learning is emphasized in \citep{neyshabur2017exploring}.  \citet{dinh2017sharp}  point out that parameter invariances can lead to the different parameterization of the same function, making the definition of flatness challenging; see also \citep{andriushchenko2023modern} for a recent study. This motivates the study of sharpness measures that are invariant to such reparametrizations.

There have been a few attempts to address reparametrization problems with sharpness measures recently. \citet{kwon2021asam} proposed to adaptively calculate the sharpness in a normalized ball around the loss function to achieve scale invariance. However, their method is limited to scaling problems. \citet{kim2022fisher} took a step further and introduced a new SAM algorithm by capturing the neighborhood of the parameters in an ellipsoid induced by the Fisher information. This way, the neighborhood becomes invariant with respect to the parameter invariances in the network. \citet{jang2022reparametrization} defined an information geometric sharpness measure by investigating the eigenspaces of Fisher Information Matrix (FIM) of distribution parameterized by neural networks. They proved scale-invariance properties for their notion. Even though~\citet{kim2022fisher} and \citet{jang2022reparametrization} enjoy some parameter invariance properties, (1) in practice, their methods are limited to classification tasks because of FIM calculation,
(2) the underlying explicit biasing of their algorithms remains a mystery and is not guaranteed.\footnote{Please refer to~\cref{app:related_work} for a more detailed overview of related work.}

\section{Background}

\subsection{Setting}

Consider a standard learning setup with a labeled dataset $\mathcal{S}$, and a training loss function $L:\mathbb{R}^d \to \mathbb{R}_{\ge0}$, where $L(x)$ denotes the training loss over $\mathcal{S}$ computed for the parameters $x \in \mathbb{R}^d$. The main objective in Empirical Risk Minimization (ERM) is to minimize the training loss $L(x)$ over the feasibility set $\mathcal{X} \subseteq \mathbb{R}^d$.  However, achieving parameters satisfying $L(x) \approx 0$ in overparameterized models is often straightforward. This is 
because in contrast to other models, in overparameterized models, there are \textit{many} global minima, i.e., the set $\Gamma:=\{x \in \mathcal{X}: L(x) = 0\}$ is a manifold -- it is called the \textit{zero-loss manifold} in the literature. 
Moreover, in practical scenarios, it is noteworthy that not all global minima exhibit favorable generalization capabilities~\citep{foret2020sharpness}.   

\subsection{Background on SAM}
It is hypothesized that the avoidance of sharp minima can help generalization performance~\citep{hochreiter1997flat,keskar2016large,izmailov2018averaging}. However, it should be noted that the concept of sharpness encompasses a multitude of distinct definitions in practical contexts. The Sharpness-Aware Minimization (SAM) algorithm \citep{foret2020sharpness} suggests minimizing the training loss function over a small ball around the parameters: 
\begin{align*}
  \min_{x \in \mathcal{X}} \Big \{ L_{\text{SAM}}(x):= \max _{\| v \|_2 \le 1} L(x + \rho v)\Big \},
\end{align*}
where $\rho \in \mathbb{R}_{\ge 0}$ is the \emph{perturbation parameter}. Note that $L_{\text{SAM}}$ can be decomposed into two terms:
\begin{align*}
    L_{\text{SAM}}(x) = \underbrace{L(x)}_{\text{empirical loss}} +  \underbrace{\max _{\| v \|_2 \le 1} \big \{ L(x + \rho v) - L(x)\big\}}_{\text{sharpness}}. 
\end{align*}
\citet{foret2020sharpness} also suggest  alternative average-based sharpness-aware objectives to use PAC bounds on the generalization error of overparameterized models; we follow the definition in \citep{wen2023sharpness}: 
\begin{align*}
    L_{\text{AVG}}(x) &:= \mathbb {E}_{v \sim \mathcal{N}(0,I) }\Big[L(x +  \frac{\rho v}{\| v\|_2})\Big] =  \underbrace{L(x)}_{\text{empirical loss}} + \underbrace{\mathbb {E}_{v \sim \mathcal{N}(0,I) }\Big[L(x +  \frac{\rho v}{\| v\|_2}) - L(x)\Big]}_{\text{sharpness}}. 
\end{align*}

\citet{wen2023sharpness} recently proved that minimizing $L_{\text{SAM}}(x)$ will lead to global minima (i.e., $L(x) \approx 0$) with small  $\lambda_{\max}(\nabla^2L(x))$. In other words, SAM is (explicitly) biased towards minimizing $\lambda_{\max}(\nabla^2L(x))$. Moreover, they show that using $L_{\text{AVG}}(x)$ biases towards minimizing $\frac{1}{d}\tr(\nabla^2L(x))$. This means that SAM measures the sharpness of a global minimum by $\lambda_{\max}(\nabla^2L(x))$, while the average-based objective uses $\frac{1}{d}\tr(\nabla^2L(x))$ to evaluate it.

In the next examples, we argue how both sharpness measures above fail to define a meaningful notion for overparameterized models. In~\Cref{example2}, a special case of problem with parameter invariances, i.e., under parameter rescalings is discussed.

\begin{example} \label{example1} 
The sharpness measures $\lambda_{\max}(\nabla^2L(x)), \lambda_{\min}(\nabla^2L(x))$ and $\tr(\nabla^2L(x))$ are conceptually meaningful when the objective function $L(x)$ is convex, therefore $\lambda_i$s are nonnegative. However, the Loss landscape of neural networks is highly nonconvex, and as a result, $\lambda_i$ can be potentially negative. Consider the toy non-convex example of $L(x_1, x_2) = \frac{1}{2}(x_1^2 - x_2^2$),
\begin{align}
\nabla^2 L = \begin{bmatrix} 1 & 0 \\ 0 & -1 \end{bmatrix}.
\end{align}
For all $x_1, x_2 \in \mathbb{R}$, we know that $\textup{tr}(\nabla^2 L) = 1 +  (-1) = 0$, which in the Trace measure of sharpness it suggests that all the points $x_1, x_2 \in \mathbb{R}$ are equally flat. Are these sharpness notions really capturing the intended concepts? For a better illustration, consider the plot of this function provided in \Cref{fig:FrobSAM_example}. 
This problem extends to other existing notions.

\end{example}

\begin{figure}[h] 
\centering
\includegraphics[width=5cm]{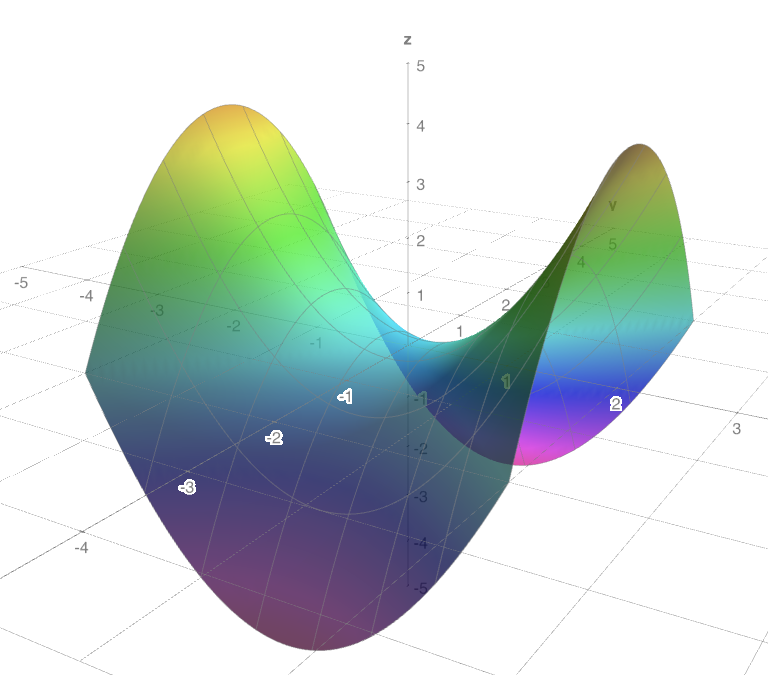}
\caption{The loss landscape of the non-convex objective function $L(x_1, x_2) = \frac{1}{2}(x_1^2 - x_2^2)$. This examples shows how existing sharpness measures fall short of capturing sharpness meaning in non-convex settings. In particular, for all points $(x_1, x_2) \in \mathbb{R}^2$,  $\textup{tr}(\nabla^2 L(x_1, x_2)) = 1 +  (-1) = 0$.}
\label{fig:FrobSAM_example}
\end{figure}

\begin{example} \label{example2}
    Consider the loss function $L(x_1,x_2)= x_1^2 x_2^2 - 2x_1x_2+1$ with two parameters $x_1,x_2\in \mathbb{R}$. It is scale-invariant, i.e., $L(kx_1,\frac{x_2}{k}) = L(x_1,x_2)$ for all $k\neq 0$. Indeed, the zero-loss manifold $\Gamma = \{(x_1,x_2): x_1x_2 = 1\}$ contains infinitely many global minima. Straightforward calculation shows $\nabla^2 L(x_1,x_2) = 
    \big(\begin{smallmatrix}
  2x_2^2 & 4x_1x_2 -2\\
  4x_1x_2-2 & 2x_1^2
\end{smallmatrix}\big)$. Thus, we have $\frac{1}{2}\tr(\nabla^2L(x_1,x_2)) = x_1^2+x_2^2$. After rescaling, we get $\frac{1}{2}\tr(\nabla^2L(x_1,x_2))\Big |_{(kx_1,k^{-1}x_2)} = k^2x_1^2+\frac{x_2^2}{k^2} \neq \frac{1}{2}\tr(\nabla^2L(x_1,x_2))$.  Therefore, as a sharpness measure, $\tr(\nabla^2L(x_1,x_2))$ is not scale-invariant. The problem magnifies in the limit:  $\tr(\nabla^2L(x_1,x_2))\Big |_{(kx_1,k^{-1}x_2)} \to  \infty$ as $k \to \infty$.
Similar problems exist for $\lambda_{\max}(\nabla^2L(x_1,x_2))$. However, $\det(\nabla^2L(x_1,x_2))$ is scale-invariant; we have $\det(\nabla^2L(x_1,x_2))\Big |_{(kx_1,k^{-1}x_2)}=\det(\nabla^2L(x_1,x_2))$ for all $k \neq 0$. 
\end{example}

Note that neural networks are often scale-invariant, e.g., linear networks or ReLU networks after scaling up the parameters of one hidden layer and scaling down the parameters of another hidden layer encode the same functions.  

\section{A New Class of Sharpness Measures}
To define a new class of sharpness measures, we take a closer look at the average-based sharpness-aware objective $L_{\text{AVG}}(x)$; using its Taylor expansion \citep{wen2023sharpness}, we have 
\begin{align*}
    L_{\text{AVG}}(x) &= \mathbb {E}_{v \sim \mathcal{N}(0,I) }\Big[L(x +  \frac{\rho v}{\| v\|_2})\Big] \\&\approx 
    L(x) +  \rho \mathbb{E}_{v \sim \mathcal{N}(0,I) } \Big[\langle \nabla L(x), \frac{v}{\| v\|_2} \rangle\Big]  + \rho^2  \mathbb{E}_{v \sim \mathcal{N}(0,I) }\Big[ \frac{v^t \nabla^2L(x) v}{\|v\|^2_2}\Big]\\& =
    L(x) + \rho^2 \frac{\tr(\nabla^2L(x))}{d}.
\end{align*}
This intuitively tells us that for a small perturbation parameter $\rho$, the leading term in the objective function is the training loss $L(x)$, and after we get close to the zero-loss manifold $\Gamma$, the leading term becomes $\frac{1}{d}\tr(\nabla^2L(x))$, which is exactly the explicit bias of the average-based sharpness-aware minimization objective. This motivates us to define the following parameterized sharpness measure.

\begin{definition}[$(\phi,\psi,\mu)$-sharpness measure] For any continuous functions $\phi,\psi:\mathbb{R}\to \mathbb{R}$ and any (Borel) measure $\mu$ on $\mathbb{R}^d$, the $(\phi,\psi,\mu)$-sharpness measure $S(x;\phi,\psi,\mu)$ is defined as  
\begin{align}
S(x;\phi,\psi,\mu):=\phi \Big( \int  \psi \big( \frac{1}{2}v^t \nabla^2L(x) v \big) d\mu(v)\Big).
\end{align}
Similarly, one can consider continuous functions $\psi:\mathbb{R}\to \mathbb{R}^m$  and $\phi:\mathbb{R}^m\to \mathbb{R}$, for some positive integer $m\ge 1$, and (Borel) measures $\mu_\ell$, $\ell \in [m]$, and define
\begin{align}
S(x;\phi,\psi,\mu):=\phi \Big(& \int  \psi_1 \big( \frac{1}{2}v^t \nabla^2L(x) v \big) d\mu_1(v),\\
& \int  \psi_2 \big( \frac{1}{2}v^t \nabla^2L(x) v \big) d\mu_2(v),\\
&\ldots,\\
&\int  \psi_m \big( \frac{1}{2}v^t \nabla^2L(x) v \big) d\mu_m(v) \Big),
\end{align}
where we use $\mu := \mu_1 \otimes \mu_2 \otimes \ldots \otimes \mu_m$ for the sake of brevity in our notation, and $\psi = (\psi_1,\psi_2,\ldots, \psi_m)^t$. 
\end{definition}
We specify several examples of hyperparameters $(\phi,\psi,\mu)$ in Table \ref{sample-table}, which shows how $(\phi,\psi,\mu)$-sharpness measures can represent various notions of sharpness, as a function of the training loss Hessian matrix.

 \begin{table*}[t]
 \setlength{\tabcolsep}{2pt}
\centering
   \caption{  Various $(\phi,\psi,\mu)$-sharpness measures (Appendix \ref{exmaple_section};   $\lambda_i$, $i \in [d]$, are eigenvalues of Hessian).}   \label{sample-table}
 \begin{tabular}{llllll}
    \toprule
    \multicolumn{4}{c}{Hyperparameters}                   \\
    \cmidrule(r){1-4}
    $\phi(t)$     & $\psi(t)$     &
    $m$   &
    $\mu$ & $S(x;\phi,\psi,\mu)$ (or bias) & Reference \\
    \midrule
    $t$ & $t$  & $1$  & $\text{Uniform}(\mathbb{S}^{d-1})$ &  $\frac{1}{2d}\tr(\nabla^2L(x)) = \frac{1}{2d}\sum_{i=1}^d\lambda_i$ &   \citep{wen2022does}   \\
    $2\max(\{.\})$ & $(t,t,\ldots)$  & $\infty$  & $\otimes_{ \|v\|_2 = 1}\delta_{v}$ &  $ \max_{i\in[d]} \lambda_i $ &   \citep{wen2022does} \\
    $2(t_2 - t_1^2)$ & $(t, t^2)$  & $2$  &  $\calN(0, I_d) \otimes \calN(0, I_d)$ &  $\sum_{i=1}^d\lambda^2_i = \| \nabla^2L(x)\|_{F}$ & This paper   \\
    $(2\pi)^{d}/t^2$     & $\exp(-t)$ & $1$  & Lebesgue measure on $\mathbb{R}^d$  &   $\det(\nabla^2L(x)) = \prod_{i=1}^d \lambda_i$ & This paper  \\
   $t$     & $t^n$        &  $1$  & $\text{Uniform}(\mathbb{S}^{d-1})$ & $q(\lambda_1,\lambda_2,\ldots,\lambda_d)$ & $\text{This paper}^{*}$ \\
   $1/t^2$     & $\exp(\sigma t)$        & $1$  &  $\calN(0,I_d)$ & $\prod_{i=1}^d(1-\sigma \lambda_i)$ & This paper \\
    \bottomrule 
  \end{tabular}
        \\ *$q(\lambda_1,\lambda_2\ldots,\lambda_d)$ is a specific  homogeneous polynomial of degree $n$; see  \cref{eq_polynomial}.
\end{table*}

\section{Expressive Power and Universality}

In this section, we prove that the proposed class of sharpness measures is \textit{universal}. In other words, for any continuous function $S:\R^{d} \to \R$, we specify continuous functions $\phi,\psi$ and a (Borel) probability measure\footnote{We indeed prove that Borel probability measures (as a subset of arbitrary Borel measures) are enough to achieve  universality. } $\mu$ on $\R^d$ such that $S(\lambda_1,\lambda_2\ldots,\lambda_d) = S(x;\phi,\psi,\mu)$, where $\lambda_i$, $i\in [d]$, are the eigenvalues of the Hessian matrix $\nabla^2 L(x)$.

\begin{restatable}[Universality of the $(\phi,\psi,\mu)$-sharpness measures for functions of Hessian eigenvalues]{theorem}{theoremuniv}\label{theoremuniv}  Let $\calA\subseteq \R^{d}$ be a compact set. For any continuous function $S: \calA \to \R$, there exist a product (Borel) probability measure $\mu$, a positive integer $m\le d$,  and continuous functions $\phi:\R^m \to \R$ and $\psi:\R \to \R^m$, such that $S(\lambda_1,\lambda_2\ldots,\lambda_d) = S(x;\phi,\psi,\mu)$ for any $x \in \calA$, where $\lambda_i$, $i\in [d]$, are the eigenvalues of the Hessian matrix $\nabla^2 L(x)$. 
\end{restatable}

We present the proof of Theorem \ref{theoremuniv} in Appendix \ref{proof_theoremuniv}.

Note that to achieve universality, we need the functions $\phi,\psi$ to be of dimension $m=d$. However, as one can see in Table \ref{sample-table}, many celebrated sharpness measures can indeed be represented using only small $m$. We believe that practically small hyperparameter $m$ is enough, as it is motivated from the measures in Table \ref{sample-table}.

While we proved the universality of the proposed class of sharpness measures for continuous functions of the Hessian eigenvalues, one may be interested in measuring sharpness with more information about the loss Hessian (e.g., the eigenvectors of the loss Hessian). The following theorem proves the universality for this class of arbitrary functions. 

\begin{restatable}[Universality of the $(\phi,\psi,\mu)$-sharpness measures for arbitrary  functions of Hessian]{theorem}{theoremuniveigenfunction}\label{theoremuniveigenfunction}  For any continuous function $S: \R^{d\times d} \to \R$, there exist  a positive integer $m\le d(d+1)/2$,  (Borel) probability measures $\mu_\ell$, $\ell \in [m]$,  and continuous functions $\phi:\R^m \to \R$ and $\psi:\R \to \R^m$, such that $S(\nabla^2 L(x)) = S(x;\phi,\psi,\mu)$ for any $x \in \R^d$, where $\mu := \mu_1\otimes \mu_2 \otimes \ldots \otimes \mu_m$ is a product probability measure. 
\end{restatable}

We present the proof of Theorem \ref{theoremuniveigenfunction} in Appendix \ref{proof_theoremuniveigenfunction}.

Note that arbitrary functions of the Hessian matrix can be quite hard to compute, e.g., consider the permanent of the Hessian matrix. Moreover, the dimension $m$ must be quite  large to allow us to prove the universality in overparameterized models (for $d$ of considerable size), since the generality bound scales as $\mathcal{O}(d^2)$. Nevertheless, in practice, only small $m$ allows to cover many interesting cases.

\section{Explicit Bias}

Now that we defined a flexible set of sharpness measures and we proved that it is  universally expressive, the following question arises: how can one achieve $S(x;\phi,\psi,\mu)$ as the explicit bias of an objective function that only relies on the zeroth-order information about the training loss, similar to $L_{\text{SAM}}(x)$ and $L_{\text{AVG}}(x)$? To answer this question, we introduce  the $(\phi,\psi,\mu)$-sharpness-aware loss function as follows.

\begin{definition}\label{def:sharp} The $(\phi,\psi,\mu)$-sharpness-aware loss function
\begin{align*}
    L_{(\phi,\psi,\mu)}&(x):= \underbrace{L(x)}_{\text{empirical loss}}  +   \rho^2 \underbrace{\phi \Big( \int  \psi \Big( 
    \frac{1}{\rho^2} \big( L(x+\rho v)- L(x) \big) 
    \Big) d\mu(v)\Big)}_{:=R_{\rho}(x)~\text{sharpness}} = L(x) + \rho^2 R_{\rho}(x),
\end{align*}
where $\rho$ is the perturbation parameter and $R_{\rho}(x)$ denotes the sharpness regularizer.
\end{definition}
  Extending this definition  to the cases with $m>1$ is straightforward.

In the above definition, the new regularizer $R_{\rho}(x)$ is an approximation of the sharpness measure $S(x;\phi,\psi,\mu)$ as $\rho \to 0^{+}$. As a result, it is expected that minimizing $L_{(\phi,\psi,\mu)}(x)$  lead to minimizing the training loss as well as the sharpness measure $S(x;\phi,\psi,\mu)$. The next theorem formalizes this intuitive observation via characterizing the explicit bias of minimizing the sharpness-aware loss function $L_{(\phi,\psi,\mu)}(x)$.

\begin{restatable}[Explicit bias of the $(\phi,\psi,\mu)$-sharpness-aware loss function]{theorem}{theorembias}\label{theorembias} Given a triplet $(\phi,\psi,\mu)$, $m\ge 1$,  and a training loss function $L:\R^d \to \R_{\ge0}$, assume that:
\begin{itemize}
    \item $L(x)$ is third-order continuously  differentiable and satisfies the following upper bound
    \begin{align}
    \max_{i,j,k \in \{1,2,3\}} |\partial_i \partial_j \partial_k L(v)|  = O(\| v\|^{-1}),
    \end{align}
    for $v \in \R^d$ as $\| v\|_2 \to \infty$. 

    \item The two functions $\phi, \psi$ are continuously differentiable. 

    \item For some $C> \max_{x \in \mathcal{X}}\max_{i \in [d]} |\lambda_{i}(\nabla^2 L(x))|$, we have 
$\int   \|v\|_2^2\tilde{\psi}_i(v)
    d\mu(v)<\infty$, $i \in [m]$, where
\begin{align}
    \tilde{\psi}_i(v) := \max_{|t|\le \|v\|_2} |\psi_i'(Ct^2)|.
\end{align}
\end{itemize}
Then, there  exists an open neighborhood $U \supsetneq \Gamma$, where $\Gamma$ is the zero-loss manifold, for connected  $U$ and $\Gamma$, such that if for some $u \in U$, one has
\begin{align}
    L(u) + \rho^2 R_{\rho}(u) - \inf_{x \in U}\Big (
    L(x) + \rho^2 R_{\rho}(x)
    \Big) \le \Delta \rho^2, 
\end{align}
with some  optimally gap $\Delta>0$, then 
\begin{align}
    L(u) \le \underbrace{
    \inf_{x \in U} L(x)
    }_{ = 0}+ (\Delta + o_{\rho}(1)) \rho^2, 
\end{align}
and also
\begin{align}
         S(u;\phi,\psi,\mu) \le  \inf_{x \in \Gamma}S(x;\phi,\psi,\mu)  + \Delta + o_{\rho}(1). 
\end{align}

\end{restatable}

We present the proof of Theorem \ref{theorembias} in Appendix \ref{proof_theorembias}.

The above theorem shows how using the new objective function $L_{(\phi,\psi,\mu)}(x)$ leads to explicitly biased optimization algorithms towards minimizing the sharpness measure $S(x;\phi,\psi,\mu)$ over the zero-loss manifold $\Gamma$. Indeed, it proves that if we are close to the zero-loss manifold (i.e., $u \in U$ for some open neighborhood $U \supsetneq \Gamma$), and also $L_{(\phi,\psi,\mu)}(u)$ is close to its global minimum over $U$, then (1) the training loss function $L(u)$ is close to zero, and (2) the corresponding sharpness measure $S(u;\phi,\psi,\mu)$ is close to its global minimum over the zero-loss manifold, with respect to an optimality gap $\Delta$.

\begin{algorithm*}[t]
   \caption{  $(\phi, \psi, \mu)$-Sharpness-Aware Minimization Algorithm (with $m = 1$)}
\begin{algorithmic}\label{algo}
   \STATE {\bfseries Input:} The triplet $(\phi, \psi, \mu)$, training loss $L(x)$, step size $\eta$, perturbation parameter $\rho$, number of samples $n$, 
   \STATE {\bfseries Output:} Model parameters $x_t$ trained with $(\phi, \psi, \mu)$-sharpness-aware minimization algorithm
   \STATE Initialization: $x \gets x_0$ and  $t \gets 0$
   \WHILE{1}
   \STATE Sample $v_1,v_2,\ldots, v_n \overset{\text{i.i.d.}}{\sim} \mu$
    \STATE Compute the following: \vspace{-0.5cm}\begin{align*}
       g_t &= \nabla L(x_t) +  \phi'\Big(\sum_{i=1}^n \frac{1}{n}\psi \Big(\frac{1}{\rho^2} \big(L(x_t+\rho v_i) - L(x_t)\big)\Big) \Big) \\
      &\hspace{2cm}\times  \sum_{i=1}^n \frac{1}{n} \Big \{
    \psi'\Big(\frac{1}{\rho^2} \big(L(x_t + \rho v_i) - L(x_t)\big) \Big)
    \times \Big(\nabla L(x_t + \rho v_i) - \nabla L(x_t)\Big) \Big \}.
    \end{align*} \vspace{-0.5cm}
    \STATE Update the parameters: $x_{t+1} = x_t - \eta g_t$
    \STATE $t \gets t+1$
   \ENDWHILE
   \end{algorithmic}
\end{algorithm*}

\section{Invariant Sharpness-Aware Minimization}

For which hyperparameters $(\phi,\psi,\mu)$ is the corresponding sharpness measure scale-invariant? The following theorem answers this question. 

\begin{restatable}[Scale-invariant $(\phi,\psi,\mu)$-sharpness measures]{theorem}{theoremscale}\label{theoremscale}  Consider a scale-invariant loss function $L(x)$ and let $\mu$ be a Borel measure  of the form
\begin{align} \label{eq:inv_measure}
    d\mu(x) = f\Big(\prod_{i=1}^d x_i\Big)\prod_{i=1}^d dx_i, 
\end{align}
where $f:\mathbb{R} \to \mathbb{R}$ is a measurable function\footnote{In \Cref{lemma:scale_invatiant_measures}, we show that any scale-invariant measure is of this form.}. Then, for any continuous functions $\phi,\psi$, the corresponding sharpness measure $S(x;\phi,\psi,\mu)$ is scale-invariant; this means that $S(x;\phi,\psi,\mu)=S(Dx;\phi,\psi,\mu)$ for any diagonal matrix $D\in \mathbb{R}^{d\times d}$ with $\det(D)=1$.
\end{restatable}

We present the proof of Theorem \ref{theoremscale} in Appendix \ref{proof_theoremscale}. 

\begin{example}
    Note that $\det(\nabla^2L(x))$ is a scale-invariant sharpness measure;  for any diagonal matrix $D\in \mathbb{R}^{d\times d}$ with $\det(D)=1$, 
    \begin{align*}
        \det(\nabla^2L(x))&\Big|_{Dx} = \det(D^{-1}\nabla^2L(x)D^{-1})  = \det(D^{-1})^2\det(\nabla^2L(x)) = \det(\nabla^2L(x)).
    \end{align*}
    Note that Theorem \ref{theoremscale} also supports the scale-invariance of the determinant; the Lebesgue measure satisfies the condition in Theorem \ref{theoremscale} with $f \equiv 1$, and we have the representation of the determinant in Table \ref{sample-table}.
\end{example}

While in \cref{theoremscale} we only considered scale-invariances, one can generalize it to a general class of parameter invariances in the following theorem.

\begin{restatable}[General parameter-invariant $(\phi,\psi,\mu)$-sharpness measures]{theorem}{theoremgroup}\label{theoremgroup}
    Let $G$ be a group acting by matrices on $\R^d$, and assume that $L(x)$ is invariant with respect to the action of $G$. Then, for any $G$-invariant (Borel) measure $\mu$, and  any continuous functions $\phi,\psi$, the corresponding sharpness measure $S(x;\phi,\psi,\mu)$ is $G$-invariant; this means that $S(x;\phi,\psi,\mu)=S(A_gx;\phi,\psi,\mu)$ for any  matrix $A_g\in \mathbb{R}^{d\times d}$ corresponding to the action of an element $g \in G$.  
\end{restatable}

The proof of this theorem is analogous to \cref{theoremscale} and is deferred to \cref{proof_theoremgroup}. Thus, the strategy to create sharpness measures invariant to any group action $G$ is simply to choose a group action invariant measure $\mu$. Now, for any family of choices of functions $\phi$ and $\psi$, we obtain a family of $G$-invariant sharpness measures. Consequently, there is a family of $G$-invariant Sharpness-Aware Minimization algorithms,  as explained in \cref{sec:algorithm}.

\section{$(\phi, \psi, \mu)$-Sharpness-Aware Minimization Algorithm} \label{sec:algorithm}

In this section, we present the pseudocode for the $(\phi, \psi, \mu)$-Sharpness-Aware Minimization Algorithm (see \cref{algo}). For simplicity, we present the algorithm for the full-batch gradient descent, and assume that $m=1$. Extending it to the mini-batch case with $m>1$ is straightforward (see \cref{algo-gen}).  The idea is to apply (stochastic) gradient decent or other optimization algorithms on the $(\phi,\psi,\mu)$-sharpness-aware loss function defined in~\cref{def:sharp},
\begin{align*}
    L_{(\phi,\psi,\mu)} = L(x) + \rho^2 R_{\rho}(x).
\end{align*}
However, calculating the sharpness term $R_{\rho}(x)$ directly is analytically hard to do because of the integration with respect to the probability measure $\mu$. Hence, we propose to estimate the inner integration at each iteration with i.i.d. random variables $\nu_1, \nu_2, \dots, \nu_n \sim \mu$ as perturbations, i.e., 
\begin{align*}
    \tilde{R}_{\rho}(x) \coloneqq \phi \Big( \frac{1}{n} \sum_{i = 1}^{n}  \psi \Big( 
    \frac{1}{\rho^2} \big( L(x+\rho v_i)- L(x) \big) 
    \Big)\Big).
\end{align*}
When $\phi$ satisfies continuity conditions, for large enough $n$, the estimator $\tilde{R}_{\rho}(x)$ will converge to $R_{\rho}(x)$. Now, we calculate the gradients of $L(x) + \rho^2 \tilde{R}_{\rho}(x)$. By chain rule,
\begin{align*}
    \rho^2 \nabla \tilde{R}_{\rho}(x)  = \phi'\Big(\sum_{i=1}^n \frac{1}{n}\psi \Big(\frac{1}{\rho^2} \big(L(x_t+\rho v_i) - L(x_t)\big)\Big) \Big)  & \times  \sum_{i=1}^n \frac{1}{n} \Big \{
    \psi'\Big(\frac{1}{\rho^2} \big(L(x_t + \rho v_i) - L(x_t)\big) \Big) \\
    & \times \Big(\nabla L(x_t + \rho v_i) - \nabla L(x_t)\Big) \Big \},
\end{align*}
which leads to \cref{algo}.

Our algorithm needs $n + 1$ gradient evaluations per iteration, which for $n=1$ matches the SAM algorithm~\citep{foret2020sharpness}. In practice, small values for $n$ demonstrate the expected results, therefore, the computational overhead of our algorithm is not a barrier.


Note that to recover the original SAM algorithm, one can set the function $\phi, \psi$ to identity, $m=1$, and choose $\mu$ to be the single-point measure on $\nabla L(x_t)/\| \nabla L(x_t)\|_2$ with $n=1$ sample  for each $t$.

Moreover, even though to prove  universality, we only used probability measures,  we proposed a compact representation of determinant with Lebesgue measure  with $m=1$ in \cref{sample-table} and \cref{theoremscale}. However,  integrals with respect to Lebesgue measure cannot be estimated via sampling and we need to truncate the integral to integration over a large hypercube; this allows us to use \cref{algo} for the scale-invariant sharpness measures. Also, this approximation achieves non-zero sharpness  in cases that the Hessian matrix is not full-rank (which happens in overparametrized models), as it gets the product of non-zero eigenvalues. We use this approximation in the next section  to implement the method. 

\section{Frobenius SAM and Determinant SAM}

To be more concrete, we specify \cref{algo-gen} (for arbitrary $m$) to the case with the Frobenius norm regularizes (with $m=2$), 
which we call the \textit{Frob-SAM} algorithm. Note that to achieve this, one needs to specify $\phi(t_1,t_2) = 2(t_2-t_1^2)$ and $\psi(t) = (t, t^2).$ Furthermore, since we only need to collect samples from the Gaussian distribution to get the Frobenius norm bias (see \cref{sample-table}), we can use the same samples to estimate both integrals for the functions $\psi_1(t) = t$ and $\psi_2(t) = t^2$. Replacing these assumptions into the formula given in \cref{algo-gen}, we get the following update rule:
\begin{align*}
    g_t = \nabla L(x_t)  +&4 \sum_{i=1}^n  \frac{1}{n\rho^2} \Big \{\big(L(x_t+\rho v_{i}) - L(x_t)\big)  \times  \Big(\nabla L(x_t + \rho v_{i}) - \nabla L(x_t)\Big) \Big \}\\& 
    -4 \Big \{\sum_{i=1}^n  
     \frac{1}{n\rho} \big(L(x_t + \rho v_{i}) - L(x_t)\big) \Big \} \times \Big \{ \sum_{j=1}^n  \frac{1}{n\rho}  \Big(\nabla L(x_t + \rho v_{j}) - \nabla L(x_t)\Big) \Big \}.
\end{align*}
If we take a closer look at this, we observe that
\begin{align*}
    g_t = \nabla L(x_t)  + \frac{4}{\rho^2} \widehat{\text{cov}}\Big (
    \big(L(x_t + \rho v) - L(x_t)\big)
    ,
    \big(\nabla L(x_t + \rho v) - \nabla L(x_t)\big) 
    \Big),
\end{align*}
where $\widehat{\text{cov}}$ denotes the (biased) empirical cross-covariance between the scalar random variable $L(x_t + \rho v) - L(x_t)$ and the vector-values random variable $\nabla L(x_t + \rho v) - \nabla L(x_t)$, for $v \sim \calN(0,I_d)$. Since the covariance is not sensitive to the means of random variables/vectors, we can further simply the update rule to
\begin{align*}
    g_t = \nabla L(x_t)  + \frac{4}{\rho^2}  \widehat{\text{cov}}\Big ( 
    L(x_t + \rho v)
    ,
    \nabla L(x_t + \rho v) 
    \Big).
\end{align*}
We can further replace the unbiased estimator of the cross-covariance instead of $\widehat{\text{cov}}$ which leads to \cref{algo-frob}.

\begin{algorithm*}[!h]
   \caption{ 
   \textit{Frob-SAM}}
\begin{algorithmic}\label{algo-frob}
   \STATE {\bfseries Input:} Training loss $L(x)$, step size $\eta$, perturbation parameter $\rho$, number of samples $n$, 
   \STATE {\bfseries Output:} Model parameters $x_t$ trained with Frobenius SAM 
   \STATE Initialization: $x \gets x_0$ and  $t \gets 0$
   \WHILE{1}
   \STATE Sample $v_{i} \overset{\text{i.i.d.}}{\sim} \calN(0,I_d)$ for any $i\in [n]$
    \STATE Compute the following: \vspace{-0.5cm} \begin{align*}
       g_t = \nabla L(x_t)  & +4 \sum_{i=1}^n  \frac{1}{(n-1)\rho^2} L(x_t+\rho v_{i}) \nabla L(x_t + \rho v_{i})  \\& 
   -4 \sum_{i=1}^n  
     \frac{1}{(n-1)\rho} L(x_t + \rho v_{i})   \times \sum_{i=1}^n  \frac{1}{n\rho}  \nabla L(x_t + \rho v_{i}).
    \end{align*}
    \vspace{-0.5cm}
    \STATE Update the parameters: $x_{t+1} = x_t - \eta g_t$
    \STATE $t \gets t+1$
   \ENDWHILE
   \end{algorithmic}
\end{algorithm*}

Moreover, achieving \textit{Det-SAM} is also similar to \textit{Frob-SAM}, but the only difficulty is that it involves computing an integral with respect to the Lebesgue measure which can be challenging (\cref{sample-table}). To address this issue, we instead sample a point from the hypercube $[-t, t]^d$ for a hyperparameter $t \in \mathbb{R}$ to approximate the Lebesgue measure.

\section{Experiments}

\begin{figure}[!t]
    \centering
    \includegraphics[width=0.48\textwidth]{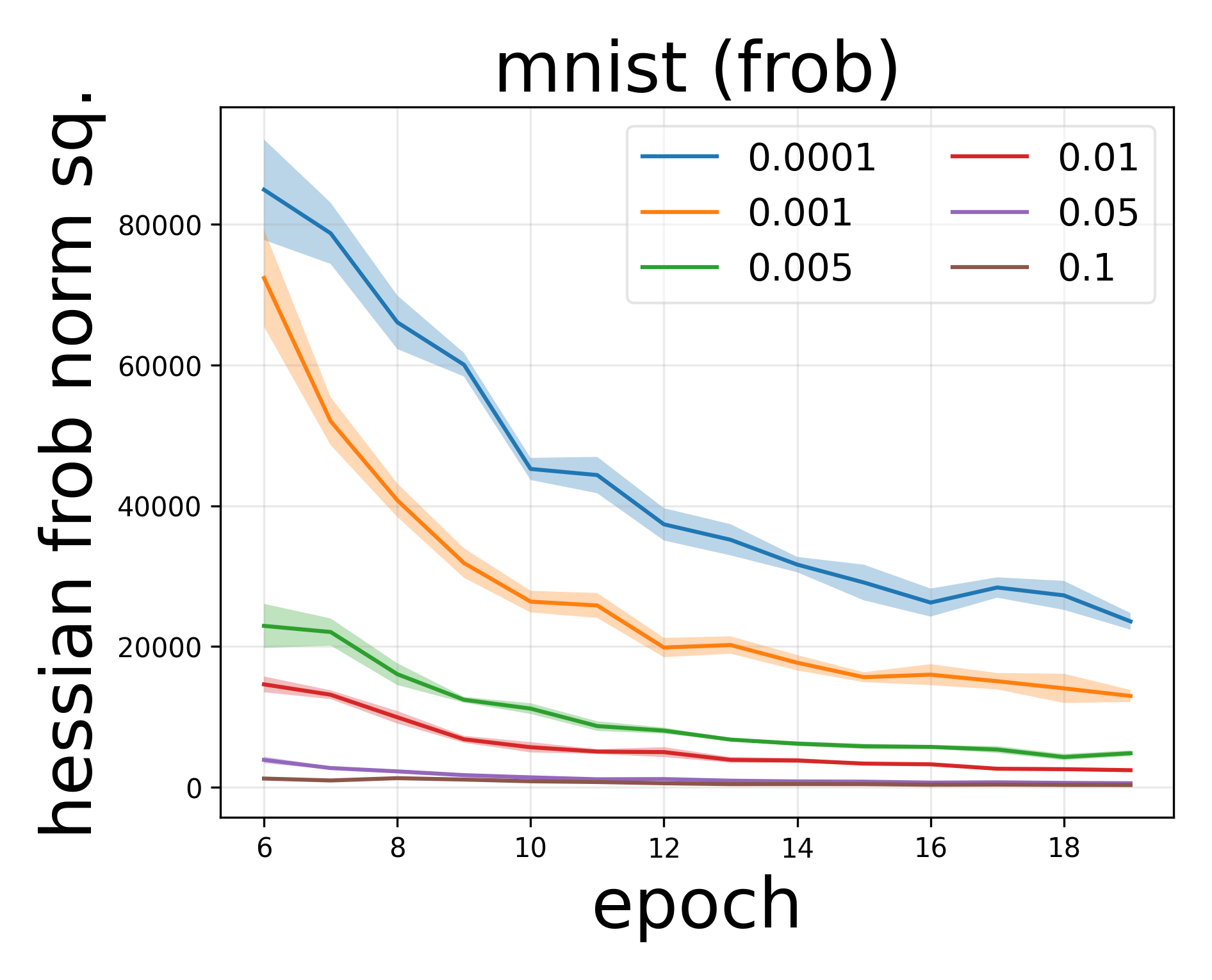}
    \caption{Sharpness measure while training on MNIST for different amounts of regularization $\lambda$. When training with Frob-SAM ($\rho = 0.01$, $n = 2$), the Frobenius norm of the Hessian decreases throughout training and larger $\lambda$ results in lower values of the norm. One standard error bar is shaded. All $\lambda$'s achieve over 94\% final test accuracy.}
    \label{fig:mnist}
\end{figure}

\begin{table*}[!t]
\setlength{\extrarowheight}{0.2em}
\centering
\begin{tabular}{r||ccc|ccc}
          & CIFAR10         & CIFAR100        & SVHN            & CIFAR10-S       & CIFAR100-S      & SVHN-S          \\ 
          \hline \hline Frob  & $94.96^{\pm0.05}$   & $77.16^{\pm0.16}$   & $96.26^{\pm0.05}$   & $\textbf{74.49}^{\pm1.15}$ & $\textbf{38.30}^{\pm0.89}$ & $89.23^{\pm0.32}$   \\
Trace & $95.04^{\pm0.07}$   & $77.61^{\pm0.08}$   & $96.37^{\pm0.04}$   & $\textbf{74.46}^{\pm0.39}$ & $\textbf{37.50}^{\pm0.78}$ & $\textbf{89.90}^{\pm0.32}$ \\
Det   & $95.10^{\pm0.03}$   & $77.64^{\pm0.12}$   & $96.32^{\pm0.05}$   & $\textbf{74.22}^{\pm0.80}$ & $\textbf{37.48}^{\pm0.42}$ & $\textbf{89.85}^{\pm0.13}$ \\
SSAM      & $\textbf{95.78}^{\pm0.05}$ & $78.42^{\pm0.19}$   & $\textbf{96.49}^{\pm0.04}$ & $73.39^{\pm0.48}$   & $35.33^{\pm0.55}$   & $\textbf{89.87}^{\pm0.27}$ \\
ASAM      & $95.49^{\pm0.09}$   & $\textbf{78.91}^{\pm0.06}$ & $96.18^{\pm0.05}$   & $\textbf{74.80}^{\pm0.64}$ & $\textbf{37.40}^{\pm0.66}$ & $89.21^{\pm0.11}$   \\
SAM       & $95.54^{\pm0.06}$   & $78.62^{\pm0.02}$   & $\textbf{96.45}^{\pm0.03}$ & $72.96^{\pm0.77}$   & $36.87^{\pm0.47}$   & $\textbf{89.80}^{\pm0.19}$ \\
SGD       & $94.81^{\pm0.07}$   & $77.00^{\pm0.11}$   & $96.06^{\pm0.08}$   & $\textbf{74.98}^{\pm0.92}$ & $\textbf{37.64}^{\pm0.88}$ & $\textbf{89.89}^{\pm0.27}$ \\ \hline
\end{tabular}
\caption{Final test accuracy and standard errors for the full and 10\% sub-sampled datasets. Entries within two standard errors of the best are bolded. Our biases particularly help when data is limited. For example, Frob-SAM achieves over 1.5\% higher test accuracy than SAM on CIFAR10-S.}
\label{table:full_and_subsampled}
\end{table*}

\begin{table}[!t]
\setlength{\extrarowheight}{0.2em}
\centering
\begin{tabular}{r||ccc}
\multicolumn{1}{l||}{} & CIFAR10-C       & CIFAR100-C      & SVHN-C          \\ \hline \hline
Frob              & $\textbf{87.57}^{\pm0.10}$ & $65.73^{\pm0.22}$   & $90.11^{\pm0.20}$   \\
Trace             & $87.20^{\pm0.12}$   & $65.29^{\pm0.12}$   & $\textbf{91.01}^{\pm0.44}$ \\
Det               & $83.61^{\pm0.06}$   & $64.73^{\pm0.29}$   & $\textbf{90.56}^{\pm0.09}$ \\
SSAM                  & $86.24^{\pm0.08}$   & $67.06^{\pm0.14}$   & $\textbf{90.92}^{\pm0.12}$ \\
ASAM                  & $85.88^{\pm0.12}$   & $\textbf{67.78}^{\pm0.04}$ & $89.54^{\pm0.12}$   \\
SAM                   & $85.42^{\pm0.08}$   & $66.12^{\pm0.13}$   & $\textbf{90.66}^{\pm0.16}$ \\
SGD                   & $82.78^{\pm0.20}$   & $64.16^{\pm0.22}$   & $87.78^{\pm0.16}$ \\
\hline
\end{tabular}
\caption{Final test accuracy and standard errors when the labels of 20\% of the training examples are corrupted. Entries within two standard errors of the best are bolded. We see that our bias can provide a boost in this scenario. For example, Frob-SAM achieves over 2\% higher accuracy than SAM on CIFAR10-C.}
\label{table:corrupt}
\end{table}

The goal of our experiments is twofold. Firstly, we validate Theorem \ref{theorembias} by showing that minimization of the sharpness-aware loss defined in Definition \ref{def:sharp} and codified in \cref{algo} has the explicit bias of minimizing the sharpness measure. Secondly, we show our method is useful in practice by evaluating it on benchmark vision tasks. Our code is available at \url{https://github.com/dbahri/universal_sam}.

\subsection{Setup}
We evaluate on three vision datasets: CIFAR10, CIFAR100, and SVHN. Futhermore, we study how our \emph{explicit} bias may be helpful in settings that generally benefit from regularization -- specifically, when \emph{training data is limited} and when \emph{training labels are noisy}. For the former, we artificially sub-sample each original dataset, keeping only the first 10\% of training samples, and we denote these sub-sampled datasets with ``-S'' (i.e. CIFAR10-S). For the latter, we choose a random 20\% of training samples to corrupt, and we corrupt these samples by flipping their labels to a different label chosen uniformly at random over the remaining classes. We denote these datasets by ``--C'' (i.e. CIFAR10-C).

Full experimental details are deferred to \cref{sec:full_experiments}; we summarize them here. We train ResNet18~\citep{he2016deep} on the datasets using momentum-SGD and a multi-step learning rate schedule. We run each experiment under four different random seeds. We evaluate three sharpness measures and the following other baselines:

\textbf{Frob-SAM}/\textbf{Trace-SAM}/\textbf{Det-SAM}. These correspond to instances of our algorithm where the measure is, respectively, the Frobenius norm, trace, and determinant of the Hessian. For Det-SAM, we set $t$, half the edge width of the approximating hypercube, to 0.01 with little tuning.

\textbf{SAM}. We set $\rho$ to 0.05/0.1/0.05 for CIFAR10/CIFAR100/SVHN, following \citet{foret2020sharpness}.

\textbf{Adaptive SAM (ASAM)}. \citet{kwon2021asam} proposes a modification of SAM that is scale-invariant. We set $\rho$ to 0.5/1/0.5 and $\eta$ to 0.01/0.1/0.01 for CIFAR10/CIFAR100/SVHN.

\textbf{Sparse SAM (SSAM)}. \citet{mi2022make} speeds up and improves the performance of SAM by only perturbing important parameters, as determined via Fisher information and sparse dynamic training. We use SSAM-F with $\rho$ set to 0.1/0.2/0.1 for CIFAR10/CIFAR100/SVHN. 50\% sparsity is used with 16 samples.

Furthermore, to demonstrate that minimization of our loss has the explicit bias we expect, we train a simple 6-layer ReLU network with 128 hidden units on MNIST using momentum-SGD (with momentum 0.9 and learning rate 0.001) for 20 epochs. We estimate the Frobenius norm of the Hessian (see \cref{sec:full_experiments} for details) throughout training for Frob-SAM, under a range of regularization strengths $\lambda$. 

\subsection{Results}
MNIST results are shown in \cref{fig:mnist}. We see that for Frob-SAM, minimization of the loss leads to reduction in the corresponding sharpness measure and the reduction expectantly scales proportionally with regularization strength $\lambda$.

Results for the main datasets are shown in \cref{table:full_and_subsampled} and \cref{table:corrupt}. We find that our method nearly always outperforms SGD and is at or sometimes above par with SAM and alternatives, especially in the noisy label and limited data settings. For example, for noisy label CIFAR10 (CIFAR10-C), Frob-SAM achieves nearly 5\% and 2\% higher final test accuracy than SGD and SAM respectively. Our findings suggest that the explicit biases we propose can be practically useful, especially in the noisy label and limited data scenarios, though it is often unclear what the \emph{best} bias for a particular task is, a priori.

\section{Conclusion}
In this paper, we introduce a new family of sharpness measures and demonstrate how this parameterized representation can generate many meaningful sharpness notions (Table \ref{sample-table}). These measures are indeed universally expressive. Furthermore, in Theorem \ref{theorembias}, we illustrate how the corresponding zeroth-order objective function for each sharpness measure is explicitly biased towards minimizing the sharpness of the training loss. Moreover, in Theorem \ref{theoremscale}, we prove how specific (Borel) measures can lead to scale-invariant sharpness measures, such as the determinant of the Hessian matrix. We conclude the paper with a series of numerical experiments showcasing the efficacy of the proposed loss function on various practical datasets. Given the broad class of sharpness measures we proposed, an interesting future direction is to evaluate in practice which sharpness measure/algorithm performs the best for different dataset. Another interesting attitude is, given the universally expressivity of our algorithm, to meta-learn the sharpness measures on the top of the model. This could potentially yield the identification of optimal, intricate sharpness measures tailored to specific datasets, thereby contributing to unraveling hidden secrets of overparameterized models.

\section*{Acknowledgments}
The authors appreciate Joshua Robinson for his insightful comments and valuable suggestions.
BT and SJ are supported by the Office of Naval Research award N00014-20-1-2023 (MURI ML-SCOPE), NSF award CCF-2112665 (TILOS AI Institute), NSF award 2134108, and the Alexander von Humboldt Foundation.
 AS and PJ  are supported by the National Research Foundation Singapore and DSO National Laboratories under the AI Singapore Programme (AISG Award No: AISG2-RP-2020-018).


\bibliography{sharpness}

\begin{thebibliography}{62}
\providecommand{\natexlab}[1]{#1}
\providecommand{\url}[1]{\texttt{#1}}
\expandafter\ifx\csname urlstyle\endcsname\relax
  \providecommand{\doi}[1]{doi: #1}\else
  \providecommand{\doi}{doi: \begingroup \urlstyle{rm}\Url}\fi

\bibitem[Agarwala and Dauphin(2023)]{agarwala2023sam}
Atish Agarwala and Yann Dauphin.
\newblock {SAM} operates far from home: eigenvalue regularization as a
  dynamical phenomenon.
\newblock In \emph{Int. Conference on Machine Learning (ICML)}, 2023.

\bibitem[Andriushchenko and Flammarion(2022)]{andriushchenko2022towards}
Maksym Andriushchenko and Nicolas Flammarion.
\newblock Towards understanding sharpness-aware minimization.
\newblock In \emph{Int. Conference on Machine Learning (ICML)}, 2022.

\bibitem[Andriushchenko et~al.(2023{\natexlab{a}})Andriushchenko, Bahri,
  Mobahi, and Flammarion]{andriushchenko2023sharpness}
Maksym Andriushchenko, Dara Bahri, Hossein Mobahi, and Nicolas Flammarion.
\newblock Sharpness-aware minimization leads to low-rank features.
\newblock \emph{arXiv preprint arXiv:2305.16292}, 2023{\natexlab{a}}.

\bibitem[Andriushchenko et~al.(2023{\natexlab{b}})Andriushchenko, Croce,
  M{\"u}ller, Hein, and Flammarion]{andriushchenko2023modern}
Maksym Andriushchenko, Francesco Croce, Maximilian M{\"u}ller, Matthias Hein,
  and Nicolas Flammarion.
\newblock A modern look at the relationship between sharpness and
  generalization.
\newblock In \emph{Int. Conference on Machine Learning (ICML)},
  2023{\natexlab{b}}.

\bibitem[Arora et~al.(2022)Arora, Li, and Panigrahi]{arora2022understanding}
Sanjeev Arora, Zhiyuan Li, and Abhishek Panigrahi.
\newblock Understanding gradient descent on the edge of stability in deep
  learning.
\newblock In \emph{Int. Conference on Machine Learning (ICML)}, 2022.

\bibitem[Azizan and Hassibi(2019)]{azizan2018stochastic}
Navid Azizan and Babak Hassibi.
\newblock Stochastic gradient/mirror descent: Minimax optimality and implicit
  regularization.
\newblock In \emph{Int. Conference on Learning Representations (ICLR)}, 2019.

\bibitem[Bahri et~al.(2022)Bahri, Mobahi, and Tay]{bahri2022sharpness}
Dara Bahri, Hossein Mobahi, and Yi~Tay.
\newblock Sharpness-aware minimization improves language model generalization.
\newblock In \emph{Proceedings of the Annual Meeting of the Association for
  Computational Linguistics}, 2022.

\bibitem[Bartlett et~al.(2022)Bartlett, Long, and
  Bousquet]{bartlett2022dynamics}
Peter~L. Bartlett, Philip~M. Long, and Olivier Bousquet.
\newblock The dynamics of sharpness-aware minimization: Bouncing across ravines
  and drifting towards wide minima.
\newblock \emph{arXiv preprint arXiv:2210.01513}, 2022.

\bibitem[Behdin and Mazumder(2023)]{behdin2023statistical}
Kayhan Behdin and Rahul Mazumder.
\newblock On statistical properties of sharpness-aware minimization: Provable
  guarantees.
\newblock \emph{arXiv preprint arXiv:2302.11836}, 2023.

\bibitem[Behdin et~al.(2022)Behdin, Song, Gupta, Durfee, Acharya, Keerthi, and
  Mazumder]{behdin2022improved}
Kayhan Behdin, Qingquan Song, Aman Gupta, David Durfee, Ayan Acharya, Sathiya
  Keerthi, and Rahul Mazumder.
\newblock Improved deep neural network generalization using m-sharpness-aware
  minimization.
\newblock \emph{arXiv preprint arXiv:2212.04343}, 2022.

\bibitem[Blanc et~al.(2020)Blanc, Gupta, Valiant, and
  Valiant]{blanc2020implicit}
Guy Blanc, Neha Gupta, Gregory Valiant, and Paul Valiant.
\newblock Implicit regularization for deep neural networks driven by an
  ornstein-uhlenbeck like process.
\newblock In \emph{Conference on Learning Theory (COLT)}, 2020.

\bibitem[Cha et~al.(2021)Cha, Chun, Lee, Cho, Park, Lee, and Park]{cha2021swad}
Junbum Cha, Sanghyuk Chun, Kyungjae Lee, Han-Cheol Cho, Seunghyun Park, Yunsung
  Lee, and Sungrae Park.
\newblock Swad: Domain generalization by seeking flat minima.
\newblock In \emph{Advances in Neural Information Processing Systems
  (NeurIPS)}, 2021.

\bibitem[Chaudhari et~al.(2019)Chaudhari, Choromanska, Soatto, Lecun, Baldassi,
  Borgs, Chayes, Sagun, and Zecchina]{chaudhari2019entropy}
Pratik Chaudhari, Anna Choromanska, Stefano Soatto, Yann Lecun, Carlo Baldassi,
  Christian Borgs, Jennifer Chayes, Levent Sagun, and Riccardo Zecchina.
\newblock Entropy-sgd: Biasing gradient descent into wide valleys.
\newblock \emph{Journal of Statistical Mechanics: Theory and Experiment},
  2019\penalty0 (12):\penalty0 124018, 2019.

\bibitem[Compagnoni et~al.(2023)Compagnoni, Biggio, Orvieto, Proske, Kersting,
  and Lucchi]{compagnoni2023sde}
Enea~Monzio Compagnoni, Luca Biggio, Antonio Orvieto, Frank~Norbert Proske,
  Hans Kersting, and Aurelien Lucchi.
\newblock An {SDE} for modeling sam: Theory and insights.
\newblock In \emph{Int. Conference on Machine Learning (ICML)}, 2023.

\bibitem[Damian et~al.(2021)Damian, Ma, and Lee]{damian2021label}
Alex Damian, Tengyu Ma, and Jason~D Lee.
\newblock Label noise {SGD} provably prefers flat global minimizers.
\newblock In \emph{Advances in Neural Information Processing Systems
  (NeurIPS)}, 2021.

\bibitem[Dinh et~al.(2017)Dinh, Pascanu, Bengio, and Bengio]{dinh2017sharp}
Laurent Dinh, Razvan Pascanu, Samy Bengio, and Yoshua Bengio.
\newblock Sharp minima can generalize for deep nets.
\newblock In \emph{Int. Conference on Machine Learning (ICML)}, 2017.

\bibitem[Du et~al.(2022)Du, Zhou, Feng, Tan, and Zhou]{du2022sharpness}
Jiawei Du, Daquan Zhou, Jiashi Feng, Vincent Tan, and Joey~Tianyi Zhou.
\newblock Sharpness-aware training for free.
\newblock In \emph{Advances in Neural Information Processing Systems
  (NeurIPS)}, 2022.

\bibitem[Foret et~al.(2021)Foret, Kleiner, Mobahi, and
  Neyshabur]{foret2020sharpness}
Pierre Foret, Ariel Kleiner, Hossein Mobahi, and Behnam Neyshabur.
\newblock Sharpness-aware minimization for efficiently improving
  generalization.
\newblock In \emph{Int. Conference on Learning Representations (ICLR)}, 2021.

\bibitem[Gunasekar et~al.(2018{\natexlab{a}})Gunasekar, Lee, Soudry, and
  Srebro]{gunasekar2018characterizing}
Suriya Gunasekar, Jason Lee, Daniel Soudry, and Nathan Srebro.
\newblock Characterizing implicit bias in terms of optimization geometry.
\newblock In \emph{Int. Conference on Machine Learning (ICML)},
  2018{\natexlab{a}}.

\bibitem[Gunasekar et~al.(2018{\natexlab{b}})Gunasekar, Lee, Soudry, and
  Srebro]{gunasekar2018implicit}
Suriya Gunasekar, Jason~D Lee, Daniel Soudry, and Nati Srebro.
\newblock Implicit bias of gradient descent on linear convolutional networks.
\newblock In \emph{Advances in Neural Information Processing Systems
  (NeurIPS)}, 2018{\natexlab{b}}.

\bibitem[He et~al.(2016)He, Zhang, Ren, and Sun]{he2016deep}
Kaiming He, Xiangyu Zhang, Shaoqing Ren, and Jian Sun.
\newblock Deep residual learning for image recognition.
\newblock In \emph{Proceedings of the IEEE conference on computer vision and
  pattern recognition}, pages 770--778, 2016.

\bibitem[Hochreiter and Schmidhuber(1997)]{hochreiter1997flat}
Sepp Hochreiter and J{\"u}rgen Schmidhuber.
\newblock Flat minima.
\newblock \emph{Neural computation}, 9\penalty0 (1):\penalty0 1--42, 1997.

\bibitem[Izmailov et~al.(2018)Izmailov, Podoprikhin, Garipov, Vetrov, and
  Wilson]{izmailov2018averaging}
Pavel Izmailov, Dmitrii Podoprikhin, Timur Garipov, Dmitry Vetrov, and
  Andrew~Gordon Wilson.
\newblock Averaging weights leads to wider optima and better generalization.
\newblock In \emph{Conference on Uncertainty in Artificial Intelligence (UAI)},
  2018.

\bibitem[Jang et~al.(2022)Jang, Lee, Park, and Noh]{jang2022reparametrization}
Cheongjae Jang, Sungyoon Lee, Frank Park, and Yung-Kyun Noh.
\newblock A reparametrization-invariant sharpness measure based on information
  geometry.
\newblock In \emph{Advances in Neural Information Processing Systems
  (NeurIPS)}, 2022.

\bibitem[Ji and Telgarsky(2019{\natexlab{a}})]{ji2018gradient}
Ziwei Ji and Matus Telgarsky.
\newblock Gradient descent aligns the layers of deep linear networks.
\newblock In \emph{Int. Conference on Learning Representations (ICLR)},
  2019{\natexlab{a}}.

\bibitem[Ji and Telgarsky(2019{\natexlab{b}})]{ji2019implicit}
Ziwei Ji and Matus Telgarsky.
\newblock The implicit bias of gradient descent on nonseparable data.
\newblock In \emph{Conference on Learning Theory (COLT)}, 2019{\natexlab{b}}.

\bibitem[Jiang et~al.(2020)Jiang, Neyshabur, Mobahi, Krishnan, and
  Bengio]{jiang2019fantastic}
Yiding Jiang, Behnam Neyshabur, Hossein Mobahi, Dilip Krishnan, and Samy
  Bengio.
\newblock Fantastic generalization measures and where to find them.
\newblock In \emph{Int. Conference on Learning Representations (ICLR)}, 2020.

\bibitem[Kaddour et~al.(2022)Kaddour, Liu, Silva, and Kusner]{kaddour2022flat}
Jean Kaddour, Linqing Liu, Ricardo Silva, and Matt~J Kusner.
\newblock When do flat minima optimizers work?
\newblock In \emph{Advances in Neural Information Processing Systems
  (NeurIPS)}, 2022.

\bibitem[Keskar et~al.(2017)Keskar, Mudigere, Nocedal, Smelyanskiy, and
  Tang]{keskar2016large}
Nitish~Shirish Keskar, Dheevatsa Mudigere, Jorge Nocedal, Mikhail Smelyanskiy,
  and Ping Tak~Peter Tang.
\newblock On large-batch training for deep learning: Generalization gap and
  sharp minima.
\newblock In \emph{Int. Conference on Learning Representations (ICLR)}, 2017.

\bibitem[Kim et~al.(2023)Kim, Park, Choi, Lee, and Lee]{kim2023exploring}
Hoki Kim, Jinseong Park, Yujin Choi, Woojin Lee, and Jaewook Lee.
\newblock Exploring the effect of multi-step ascent in sharpness-aware
  minimization.
\newblock \emph{arXiv preprint arXiv:2302.10181}, 2023.

\bibitem[Kim et~al.(2022)Kim, Li, Hu, and Hospedales]{kim2022fisher}
Minyoung Kim, Da~Li, Shell~X Hu, and Timothy Hospedales.
\newblock Fisher sam: Information geometry and sharpness aware minimisation.
\newblock In \emph{Int. Conference on Machine Learning (ICML)}, 2022.

\bibitem[Kwon et~al.(2021)Kwon, Kim, Park, and Choi]{kwon2021asam}
Jungmin Kwon, Jeongseop Kim, Hyunseo Park, and In~Kwon Choi.
\newblock {ASAM}: Adaptive sharpness-aware minimization for scale-invariant
  learning of deep neural networks.
\newblock In \emph{Int. Conference on Machine Learning (ICML)}, 2021.

\bibitem[Lawrence et~al.(2022)Lawrence, Georgiev, Dienes, and
  Kiani]{lawrence2021implicit}
Hannah Lawrence, Kristian Georgiev, Andrew Dienes, and Bobak~T Kiani.
\newblock Implicit bias of linear equivariant networks.
\newblock In \emph{Int. Conference on Machine Learning (ICML)}, 2022.

\bibitem[Le and Jegelka(2022)]{le2022training}
Thien Le and Stefanie Jegelka.
\newblock Training invariances and the low-rank phenomenon: beyond linear
  networks.
\newblock In \emph{Int. Conference on Learning Representations (ICLR)}, 2022.

\bibitem[Li et~al.(2022{\natexlab{a}})Li, Wang, and Arora]{li2021happens}
Zhiyuan Li, Tianhao Wang, and Sanjeev Arora.
\newblock What happens after sgd reaches zero loss?--a mathematical framework.
\newblock In \emph{Int. Conference on Learning Representations (ICLR)},
  2022{\natexlab{a}}.

\bibitem[Li et~al.(2022{\natexlab{b}})Li, Wang, and Yu]{li2022fast}
Zhiyuan Li, Tianhao Wang, and Dingli Yu.
\newblock Fast mixing of stochastic gradient descent with normalization and
  weight decay.
\newblock In \emph{Advances in Neural Information Processing Systems
  (NeurIPS)}, 2022{\natexlab{b}}.

\bibitem[Liang et~al.(2019)Liang, Poggio, Rakhlin, and Stokes]{liang2019fisher}
Tengyuan Liang, Tomaso Poggio, Alexander Rakhlin, and James Stokes.
\newblock Fisher-rao metric, geometry, and complexity of neural networks.
\newblock In \emph{Int. Conference on Artificial Intelligence and Statistics
  (AISTATS)}, 2019.

\bibitem[Lim et~al.(2023)Lim, Robinson, Zhao, Smidt, Sra, Maron, and
  Jegelka]{lim2022sign}
Derek Lim, Joshua Robinson, Lingxiao Zhao, Tess Smidt, Suvrit Sra, Haggai
  Maron, and Stefanie Jegelka.
\newblock Sign and basis invariant networks for spectral graph representation
  learning.
\newblock In \emph{Int. Conference on Learning Representations (ICLR)}, 2023.

\bibitem[Liu et~al.(2022{\natexlab{a}})Liu, Mai, Chen, Hsieh, and
  You]{liu2022towards}
Yong Liu, Siqi Mai, Xiangning Chen, Cho-Jui Hsieh, and Yang You.
\newblock Towards efficient and scalable sharpness-aware minimization.
\newblock In \emph{IEEE Conference on Computer Vision and Pattern Recognition
  (CVPR)}, 2022{\natexlab{a}}.

\bibitem[Liu et~al.(2022{\natexlab{b}})Liu, Mai, Cheng, Chen, Hsieh, and
  You]{liu2022random}
Yong Liu, Siqi Mai, Minhao Cheng, Xiangning Chen, Cho-Jui Hsieh, and Yang You.
\newblock Random sharpness-aware minimization.
\newblock In \emph{Advances in Neural Information Processing Systems
  (NeurIPS)}, 2022{\natexlab{b}}.

\bibitem[Long and Bartlett(2023)]{long2023sharpnessaware}
Philip~M. Long and Peter~L. Bartlett.
\newblock Sharpness-aware minimization and the edge of stability.
\newblock \emph{arXiv preprint arXiv:2309.12488}, 2023.

\bibitem[Lu et~al.(2022)Lu, Kobyzev, Rezagholizadeh, Rashid, Ghodsi, and
  Langlais]{lu2022improving}
Peng Lu, Ivan Kobyzev, Mehdi Rezagholizadeh, Ahmad Rashid, Ali Ghodsi, and
  Philippe Langlais.
\newblock Improving generalization of pre-trained language models via
  stochastic weight averaging.
\newblock \emph{arXiv preprint arXiv:2212.05956}, 2022.

\bibitem[Lyu et~al.(2022)Lyu, Li, and Arora]{lyu2022understanding}
Kaifeng Lyu, Zhiyuan Li, and Sanjeev Arora.
\newblock Understanding the generalization benefit of normalization layers:
  Sharpness reduction.
\newblock In \emph{Advances in Neural Information Processing Systems
  (NeurIPS)}, 2022.

\bibitem[Madry et~al.(2018)Madry, Makelov, Schmidt, Tsipras, and
  Vladu]{madry2017towards}
Aleksander Madry, Aleksandar Makelov, Ludwig Schmidt, Dimitris Tsipras, and
  Adrian Vladu.
\newblock Towards deep learning models resistant to adversarial attacks.
\newblock In \emph{Int. Conference on Learning Representations (ICLR)}, 2018.

\bibitem[Mi et~al.(2022)Mi, Shen, Ren, Zhou, Sun, Ji, and Tao]{mi2022make}
Peng Mi, Li~Shen, Tianhe Ren, Yiyi Zhou, Xiaoshuai Sun, Rongrong Ji, and
  Dacheng Tao.
\newblock Make sharpness-aware minimization stronger: A sparsified perturbation
  approach.
\newblock In \emph{Advances in Neural Information Processing Systems
  (NeurIPS)}, 2022.

\bibitem[Neyshabur et~al.(2017)Neyshabur, Bhojanapalli, McAllester, and
  Srebro]{neyshabur2017exploring}
Behnam Neyshabur, Srinadh Bhojanapalli, David McAllester, and Nati Srebro.
\newblock Exploring generalization in deep learning.
\newblock In \emph{Advances in Neural Information Processing Systems
  (NeurIPS)}, 2017.

\bibitem[Nitanda et~al.(2023)Nitanda, Kikuchi, and Maeda]{nitanda2023parameter}
Atsushi Nitanda, Ryuhei Kikuchi, and Shugo Maeda.
\newblock Parameter averaging for sgd stabilizes the implicit bias towards flat
  regions.
\newblock \emph{arXiv preprint arXiv:2302.09376}, 2023.

\bibitem[Qu et~al.(2022)Qu, Li, Duan, Liu, Tang, and Lu]{qu2022generalized}
Zhe Qu, Xingyu Li, Rui Duan, Yao Liu, Bo~Tang, and Zhuo Lu.
\newblock Generalized federated learning via sharpness aware minimization.
\newblock In \emph{Int. Conference on Machine Learning (ICML)}, 2022.

\bibitem[Shi et~al.(2023)Shi, Liu, Wei, Shen, Wang, and Tao]{shi2023make}
Yifan Shi, Yingqi Liu, Kang Wei, Li~Shen, Xueqian Wang, and Dacheng Tao.
\newblock Make landscape flatter in differentially private federated learning.
\newblock In \emph{IEEE Conference on Computer Vision and Pattern Recognition
  (CVPR)}, 2023.

\bibitem[Soudry et~al.(2018)Soudry, Hoffer, Nacson, Gunasekar, and
  Srebro]{soudry2018implicit}
Daniel Soudry, Elad Hoffer, Mor~Shpigel Nacson, Suriya Gunasekar, and Nathan
  Srebro.
\newblock The implicit bias of gradient descent on separable data.
\newblock \emph{Journal of Machine Learning Research}, 2018.

\bibitem[Sun et~al.(2023)Sun, Shen, Zhong, Ding, Chen, Sun, Li, Sun, and
  Tao]{sun2023adasam}
Hao Sun, Li~Shen, Qihuang Zhong, Liang Ding, Shixiang Chen, Jingwei Sun, Jing
  Li, Guangzhong Sun, and Dacheng Tao.
\newblock Ada{SAM}: Boosting sharpness-aware minimization with adaptive
  learning rate and momentum for training deep neural networks.
\newblock \emph{arXiv preprint arXiv:2303.00565}, 2023.

\bibitem[Wang et~al.(2023)Wang, Zhang, Lei, and Zhang]{wang2023sharpness}
Pengfei Wang, Zhaoxiang Zhang, Zhen Lei, and Lei Zhang.
\newblock Sharpness-aware gradient matching for domain generalization.
\newblock In \emph{IEEE Conference on Computer Vision and Pattern Recognition
  (CVPR)}, 2023.

\bibitem[Wen et~al.(2023{\natexlab{a}})Wen, Ma, and Li]{wen2022does}
Kaiyue Wen, Tengyu Ma, and Zhiyuan Li.
\newblock How does sharpness-aware minimization minimize sharpness?
\newblock In \emph{Int. Conference on Learning Representations (ICLR)},
  2023{\natexlab{a}}.

\bibitem[Wen et~al.(2023{\natexlab{b}})Wen, Ma, and Li]{wen2023sharpness}
Kaiyue Wen, Tengyu Ma, and Zhiyuan Li.
\newblock Sharpness minimization algorithms do not only minimize sharpness to
  achieve better generalization.
\newblock \emph{arXiv preprint arXiv:2307.11007}, 2023{\natexlab{b}}.

\bibitem[Woodworth et~al.(2020)Woodworth, Gunasekar, Lee, Moroshko, Savarese,
  Golan, Soudry, and Srebro]{woodworth2020kernel}
Blake Woodworth, Suriya Gunasekar, Jason~D Lee, Edward Moroshko, Pedro
  Savarese, Itay Golan, Daniel Soudry, and Nathan Srebro.
\newblock Kernel and rich regimes in overparametrized models.
\newblock In \emph{Conference on Learning Theory (COLT)}, 2020.

\bibitem[Xu et~al.(2019)Xu, Hu, Leskovec, and Jegelka]{xu2018powerful}
Keyulu Xu, Weihua Hu, Jure Leskovec, and Stefanie Jegelka.
\newblock How powerful are graph neural networks?
\newblock In \emph{Int. Conference on Learning Representations (ICLR)}, 2019.

\bibitem[Yao et~al.(2020)Yao, Gholami, Keutzer, and Mahoney]{yao2020pyhessian}
Zhewei Yao, Amir Gholami, Kurt Keutzer, and Michael~W Mahoney.
\newblock Pyhessian: Neural networks through the lens of the hessian.
\newblock In \emph{2020 IEEE international conference on big data (Big data)},
  pages 581--590. IEEE, 2020.

\bibitem[Zaheer et~al.(2017)Zaheer, Kottur, Ravanbakhsh, Poczos, Salakhutdinov,
  and Smola]{zaheer2017deep}
Manzil Zaheer, Satwik Kottur, Siamak Ravanbakhsh, Barnabas Poczos, Russ~R
  Salakhutdinov, and Alexander~J Smola.
\newblock Deep sets.
\newblock In \emph{Advances in Neural Information Processing Systems
  (NeurIPS)}, 2017.

\bibitem[Zhao et~al.(2022)Zhao, Zhang, and Hu]{zhao2022randomized}
Yang Zhao, Hao Zhang, and Xiuyuan Hu.
\newblock Randomized sharpness-aware training for boosting computational
  efficiency in deep learning.
\newblock \emph{arXiv preprint arXiv:2203.09962}, 2022.

\bibitem[Zhong et~al.(2022)Zhong, Ding, Shen, Mi, Liu, Du, and
  Tao]{zhong2022improving}
Qihuang Zhong, Liang Ding, Li~Shen, Peng Mi, Juhua Liu, Bo~Du, and Dacheng Tao.
\newblock Improving sharpness-aware minimization with fisher mask for better
  generalization on language models.
\newblock \emph{arXiv preprint arXiv:2210.05497}, 2022.

\bibitem[Zhu et~al.(2023)Zhu, Wang, Wang, Zhou, and Ge]{zhu2022understanding}
Xingyu Zhu, Zixuan Wang, Xiang Wang, Mo~Zhou, and Rong Ge.
\newblock Understanding edge-of-stability training dynamics with a minimalist
  example.
\newblock In \emph{Int. Conference on Learning Representations (ICLR)}, 2023.

\bibitem[Zhuang et~al.(2022)Zhuang, Gong, Yuan, Cui, Adam, Dvornek, Tatikonda,
  Duncan, and Liu]{zhuang2022surrogate}
Juntang Zhuang, Boqing Gong, Liangzhe Yuan, Yin Cui, Hartwig Adam, Nicha
  Dvornek, Sekhar Tatikonda, James Duncan, and Ting Liu.
\newblock Surrogate gap minimization improves sharpness-aware training.
\newblock In \emph{Int. Conference on Learning Representations (ICLR)}, 2022.

\end{thebibliography}
\bibliographystyle{plainnat}

\newpage

\appendix

\onecolumn
\section{Additional Related Work} \label{app:related_work}

SAM can provide a strong regularization of the eigenvalues throughout the learning trajectory \citep{agarwala2023sam}.  \citet{bartlett2022dynamics} show that the dynamics of SAM similar to GD on the spectral norm of Hessian. \citet{compagnoni2023sde} propose an SDE for modeling SAM, while \citet{behdin2023statistical} study the statistical benefits of SAM (see also \citep{li2021happens} for a  general framework for the dynamics of SGD around the zero-loss manifold).  It is shown that SAM can reduce the feature rank (i.e., allowing learning low-rank features) \citep{andriushchenko2023sharpness}. \citet{blanc2020implicit} proved that SGD is implicitly biased toward minimizing the trace of Hessian.

\citet{kim2023exploring} proposes a multi-step ascent approach to improve SAM, while \citet{mi2022make} suggested sparsification of SAM. \citet{zhuang2022surrogate} improve  SAM   by changing the directions in the ascent step; their method is called  Surrogate Gap Guided Sharpness-Aware Minimization (GSAM) (see also \citep{behdin2022improved}). 
Random smoothing-based SAM (R-SAM) is another SAM variant that is proposed to reduce its computational complexity \citep{liu2022random} (see also \citep{du2022sharpness, liu2022towards, zhao2022randomized, sun2023adasam} for more).  Adaptive SAM (ASAM) is proposed for applying SAM on scale-invariant neural networks and has shown generalization benefits \citep{kwon2021asam}. \citet{li2022fast} also prove that scale-invariant loss functions allow faster mixing in function spaces for neural networks. \citet{lyu2022understanding}  show how normalization can make GD reduce the sharpness via a continuous sharpness-reduction flow. \citet{liang2019fisher} propose a capacity measure based on information geometry for parameter invariances in overparameterized models (for more on information geometry, see \citep{kim2022fisher, jang2022reparametrization}). \citet{jiang2019fantastic} empirically compare different complexity measures for overparameterized models.   \citet{keskar2016large} show how a large batch yields sharp minima but a small batch achieves flat minima.

Stochastic Weight Averaging (SWA) is another way to improve generalization and it relies on finding wider minima by averaging multiple points along the trajectory of SGD  \citep{izmailov2018averaging}.  (see e.g., \citep{lu2022improving} which uses this method for language models). See \citep{kaddour2022flat} for the empirical comparison between two popular flat-minima optimization approaches: SWA and SAM.

Learning with group invariant architectures has recently gained a lot of interest due to its applications in physics and biology; see e.g., deep sets \citep{zaheer2017deep}, Graph Neural Networks (GNNs) \citep{xu2018powerful}, and also sign-flips for spectral data \citep{lim2022sign}. These architectures are all owing their practical success to their specific parameter invariance.

It is also worth mentioning that neural networks trained with large learning rates often generalize better (the edge-of-stability regime); see \citep{arora2022understanding, zhu2022understanding, long2023sharpnessaware} for the theoretical understanding of this phenomenon.

\section{Examples of $(\phi,\psi,\mu)$-Sharpness Measures}\label{exmaple_section}
In this section, we prove various notions of sharpness can be achieved using the proposed approach in this paper (Table \ref{sample-table}). For the last row of Table \ref{sample-table}, we refer the reader to the proof of Theorem \ref{theoremuniv}.

\begin{itemize}
    \item \textbf{Trace.} Let $\phi(t)=\psi(t)=t$, and note that 
\begin{align}
    S(x;\phi,\psi,\mu)=&\int   \frac{1}{2}v^t \nabla^2L(x) v   d\mu(v) \\
    & = \frac{1}{2}\mathbb{E}_{v \sim \mu} [   v^t \nabla^2L(x)v], 
\end{align}
where $\mu$ is the uniform distribution over the $(d-1)$-sphere $\mathbb{S}^{(d-1)}:= \{ x \in \mathbb{R}^d: \|x\|_2 = 1\}$. Denote the entries of $\nabla^2L(x)$ as $(\nabla^2L(x))_{i,j}$. Then, by the linearity of expectation 
\begin{align}
    \mathbb{E}_{v \sim \mu} [   v^t \nabla^2L(x)] &= \sum_{i,j=1}^d (\nabla^2L(x))_{i,j}\mathbb{E}[v_iv_j] = \sum_{i=1}^d \frac{1}{d}(\nabla^2L(x))_{i,i}= \frac{1}{d}\tr(\nabla^2L(x)), 
\end{align}
since $\mathbb{E}[v_iv_j]  = \frac{1}{d} \delta_{i,j}$, where $\delta_{i,j}$ denotes the Knocker delta function.

\item \textbf{Determinant.} To achieve the determinant, we choose $\phi(t) = (2\pi)^{d}/t^2$ and $\psi(t) = \exp(-t)$. Then,
\begin{align}
    S(x;\phi,\psi,\mu)=& (2\pi)^d\Big (\int   \exp(-\frac{1}{2}v^t \nabla^2L(x) v)   dv \Big)^{-2}, 
\end{align}
where $dv$ denotes the Lebesgue measure. However, using the multivariate Gaussian integral, we have
\begin{align}
    \int   \exp\Big(-\frac{1}{2}v^t \nabla^2L(x) v\Big)   dv = (2\pi)^{d/2} \det(\nabla^2L(x))^{-1/2}.
\end{align}
Replacing this intro the definition of $S(x;\phi,\psi,\mu)$ gives the desired result.

\item 
\textbf{Polynomials of eigenvalues.} First assume that  $\psi(t) = t^n$ for some $n\ge 0$. Then,  for any function $\phi(t)$,
\begin{align}
    S(x;\phi,\psi,\mu)&=\phi \Big( \int   \big( \frac{1}{2} v^t \nabla^2L(x) v \big)^n d\mu(v)\Big) \\
    & = \phi \Big(\mathbb{E}_{v \sim \mu}\Big[  \big( \frac{1}{2}v^t \nabla^2L(x) v \big)^n \Big]\Big),
\end{align}
where $\mu$ is the uniform distribution over the $(d-1)$-sphere $\mathbb{S}^{(d-1)}$. Since $\nabla^2L(x)$ is a symmetric matrix, we can find an orthogonal matrix $Q$ such that $\nabla^2L(x) = Q^t D Q$, where $D$ is a diagonal matrix with diagonal entries $\lambda_1,\lambda_2,\ldots,\lambda_d$. Now we write 
$    \big( v^t \nabla^2L(x) v \big)^n = \big( v^tQ^t D Qv \big)^n$. But $Qv$ is distributed uniformly over the $(d-1)$-sphere $\mathbb{S}^{(d-1)}$, similar to $v$. Thus, we conclude
\begin{align}
    S(x;\phi,\psi,\mu)&=\phi \Big(\mathbb{E}_{v \sim \mu}\Big[  \big( \frac{1}{2}v^t \nabla^2L(x) v \big)^n \Big]\Big) \\
    & = \phi \Big(\mathbb{E}_{v \sim \mu}\Big[  \big( \frac{1}{2}\sum_{i=1}^d \lambda_i v_i^2 \big)^n \Big]\Big).
\end{align}
Define \begin{align}
    q(\lambda_1,\lambda_2,\ldots,\lambda_d):= \mathbb{E}_{v \sim \mu}\Big[  \big( \sum_{i=1}^d \frac{1}{2}\lambda_i v_i^2 \big)^n \Big], \label{eq_polynomial}
\end{align}
which is clearly a polynomial function (by the linearity of expectation).

Note that the above computation is still valid  if we replace the uniform distribution on hypersphere with the Gaussian multivariate distribution with identity covariance $\calN(0, I_d)$.  Indeed, let us compute this polynomial for $n = 2$ with Gaussian distribution. Note that 
\begin{align}
q(\lambda_1,\lambda_2,\ldots,\lambda_d)&= \mathbb{E}_{v \sim \mu}\Big[  \big( \sum_{i=1}^d \frac{1}{2}\lambda_i v_i^2 \big)^2 \Big] = \frac{1}{4}\sum_{i=1}^d\lambda_i^2 \E[Z^4]+ \sum_{i\neq j} \frac{1}{4}\lambda_i \lambda_j (\E[Z^2])^2\\
& = \frac{3}{4}\sum_{i=1}^d\lambda_i^2+ \frac{1}{4}\sum_{i\neq j} \lambda_i \lambda_j,
\end{align}
where $Z$ is a zero-mean Gaussian random variable with unit variance, and note that $E[Z^2] = 1$ and $E[Z^4] = 3$. 

Now if we take $m=2$, and $\psi(t) = (t, t^2)$, with $\mu = \calN(0,I_d) \otimes \calN(0,I_d)$,  we have that 
\begin{align}
     \int   \psi \big(\frac{1}{2}v^t \nabla^2L(x) v \big) d\mu(v) =  \Big( \frac{1}{2}\sum_{i=1}^d\lambda_i, \frac{3}{4}\sum_{i=1}^d\lambda_i^2+ \frac{1}{4}\sum_{i\neq j} \lambda_i \lambda_j \Big).
\end{align}
Finally, by taking $\phi(t_1,t_2) = 2(t_2-t_1^2)$, we obtain 
\begin{align}
\phi\Big(\int   \psi \big( \frac{1}{2}v^t \nabla^2L(x) v \big) d\mu(v)\Big) =   \sum_{i=1}^d\lambda_i^2. 
\end{align}
\end{itemize}

\section{Proof of Theorem \ref{theoremuniv}}\label{proof_theoremuniv}

\theoremuniv*

\begin{proof}

We explicitly construct the (Borel) probability measure $\mu$ and the function $\psi: \R^m \to \R$. Let us take $m=d$ to prove the universality theorem, while we believe lower $m$ should be enough for specific practical sharpness measures.

Indeed, let us consider $\mu$ to be the multivariate Gaussian probability measure with identity covariance matrix. Also, define a (parameterized) function $\psi_{\sigma}(t) = \exp(\sigma t)$, for some $\sigma$ to be set later.  We are interested to compute the following quantity:
\begin{align}
    \int  \psi_{\sigma} \big( \frac{1}{2}v^t \nabla^2L(x) v \big) d\mu(v).
\end{align}
Note that  $\mu$ is the standard Gaussian probability measure on $\R^d$, and specifically, it's invariant under the action of the  orthogonal matrices. Indeed, there exists an orthogonal matrix $Q$ such that $\nabla^2L(x) = Q^t \Lambda Q$, where $\Lambda$ is a diagonal matrix with diagonal entries $\lambda_1,\lambda_2,\ldots,\lambda_d$. Observe that $u:=Qv$ is also distributed according to the Gaussian probability distribution with the identity covariance matrix on $\R^d$, similar to $v$. Now we write 
\begin{align}
     \int  \psi_{\sigma} \big( \frac{1}{2}v^t \nabla^2L(x) v \big) d\mu(v) &= 
     \int  \psi_{\sigma} \big( \frac{1}{2}v^t Q^t \Lambda Q v \big) d\mu(v) \\
     &= \int  \psi_{\sigma} \big( \frac{1}{2} u^t  \Lambda u \big) d\mu(u)\\&
     = \int  \psi_{\sigma} \Big(\frac{1}{2}\sum_{i=1}^d \lambda_i u_i^2 \Big) d\mu(u) \\
     & = \int (2\pi)^{-d/2} \psi_{\sigma} \Big(\frac{1}{2}\sum_{i=1}^d \lambda_i u_i^2 \Big) \exp \Big(-\frac{1}{2}\sum_{i=1}^d u_i^2\Big)du \\
     & = \int (2\pi)^{-d/2} \exp \Big(\frac{1}{2}\sum_{i=1}^d \sigma \lambda_i u_i^2\Big) \exp \Big(-\frac{1}{2}\sum_{i=1}^d u_i^2\Big)du\\&
      = \int (2\pi)^{-d/2} \exp \Big( \sum_{i=1}^d \frac{1}{2}(\sigma \lambda_i-1)u_i^2\Big)du,
\end{align}
where $du$ denotes the Lebesgue measure on $\R^d$. Now to compute the integral, note that $du = du_1 \times  du_2 \times \ldots \times du_d$ is a product measure and the integrand also takes on a product form; thus,
\begin{align}
     \int  \psi_{\sigma} \big( \frac{1}{2}v^t \nabla^2L(x) v \big) d\mu(v) &= 
       \int (2\pi)^{-d/2} \exp \Big( \sum_{i=1}^d \frac{1}{2}(\sigma \lambda_i-1)u_i^2\Big)du \\&
      = \prod_{i=1}^d\int (2\pi)^{-1/2} \exp \Big( \frac{1}{2}(\sigma \lambda_i-1)u_i^2\Big)du_i\\&
      \overset{(a)}{=} \prod_{i=1}^d \frac{1}{\sqrt{1-\sigma\lambda_i}}
      \underbrace{\int \sqrt{ \frac{1-\sigma\lambda_i}{2\pi} }\exp \Big( \frac{1}{2}(\sigma \lambda_i-1)u_i^2\Big)du_i}_{=1}\\
      &= \prod_{i=1}^d\frac{1}{\sqrt{1-\sigma\lambda_i}},
\end{align}
where (a) holds by the Gaussian integral identities. 

Note that to calculate the integral above, we assumed that $1-\sigma\lambda_i >0$ for any $i \in [d]$. This is equivalent to having
\begin{align}
     \max_{\lambda_i <0} \lambda^{-1}_i < \sigma < \min_{\lambda_i >0} \lambda^{-1}_i.
\end{align}
Now, due to the assumption in the theorem,  we study the target sharpness measure $S(\lambda_1,\lambda_2,\ldots, \lambda_d)$ only on a compact domain $\calA \subseteq \R^d$, and this means that  there exists an open interval $I = (-\epsilon, \epsilon)$ with 
\begin{align}
    \epsilon: = \min _{\big(\lambda_1,\lambda_2, \ldots, \lambda_d \big)^t \in \calA}~  \min_{i} |\lambda_i|^{-1},
\end{align}
such that the above integral is well-defined and finite for all $\sigma \in I$. 

Let us define the function $\tilde{\phi}: \R^d \to \R^d$ as follows:
\begin{align}
    \tilde{\phi}(t_1,t_2,\ldots,t_d):= (t_1^{-2}, t_2^{-2},\ldots, t_d^{-2}).
\end{align}
Finally, consider the following polynomial in one variable with degree $d$:
\begin{align}
 p(x):= (1-\lambda_1 x)\times (1-\lambda_2 x)\times\cdots \times (1-\lambda_d x).
\end{align}

\begin{claim} For any $\sigma_i \in I$, $i \in [d]$, we have
\begin{align}
    \tilde{\phi}\Big (
    \int  \psi_{\sigma_1} \big(\frac{1}{2} v^t \nabla^2L(x) v \big) d\mu(v),
    \int  \psi_{\sigma_2} \big(\frac{1}{2} v^t \nabla^2L(x) v \big) d\mu(v),
    \ldots,&
    \int  \psi_{\sigma_d} \big( \frac{1}{2}v^t \nabla^2L(x) v \big) d\mu(v)
    \Big) &\\
    = \Big(p(\sigma_1), p(\sigma_2), \ldots, p(\sigma_d)\Big).
\end{align}
\end{claim}
The above claim simply follows from the integral we calculated before.

Now we are ready to complete the proof. Choose arbitrary non-zero distinct $\sigma_i\in I, i\in [d]$, and note that having access to $\big(p(\sigma_1), p(\sigma_2), \ldots, p(\sigma_d)\big)$ is enough to recover all the eigenvalues. Indeed, assume  that $p(x) = p_0 + p_1x + p_2x^2 +\ldots + p_d x^d$ and note that $p(0)=1$ by definition. Let also $V(0, \sigma_1,\sigma_2,\ldots, \sigma_d) \in \R^{d \times d}$ denote a Vandermonde matrix of order $d+1$, which is provably invertible by definition, and note that
\begin{align}
    \big(p(\sigma_1), p(\sigma_2), &\ldots, p(\sigma_d)\big)^t = V(0, \sigma_1,\sigma_2,\ldots, \sigma_d)   \times \big(p_0, p_1, \ldots, p_{d+1}\big)^t \\
    &\implies \big(p_0, p_1, \ldots, p_{d+1}\big)^t =  V(0, \sigma_1,\sigma_2,\ldots, \sigma_d)^{-1} \times \big(p(\sigma_1), p(\sigma_2), \ldots, p(\sigma_d)\big)^t. 
\end{align}
Indeed, this shows that having access to the vector $\big(p(\sigma_1), p(\sigma_2), \ldots, p(\sigma_d)\big)^t$ is enough to reconstruct the polynomial $p(x) = p_0 + p_1x + \ldots + p_d x^d$. Having access to this polynomial is equivalent to having access to its roots, so one can find a continuous function $\phi_1: \R^d \to \R^d$ such that \begin{align*}
    \big(&\lambda_1,\lambda_2, \ldots,\lambda_d\big)^t \\
    &= \phi_1 \circ \tilde{\phi}\Big (
    \int  \psi_{\sigma_1} \big(\frac{1}{2} v^t \nabla^2L(x) v \big) d\mu(v),
    \int  \psi_{\sigma_2} \big( \frac{1}{2}v^t \nabla^2L(x) v \big) d\mu(v),
    \ldots,
    \int  \psi_{\sigma_d} \big( \frac{1}{2}v^t \nabla^2L(x) v \big) d\mu(v)
    \Big).
\end{align*}
Since the sharpness measure $S(\lambda_1,\lambda_2,\ldots, \lambda_d)$ is a continuous function of its coordinates, we conclude that 
\begin{align*}
     &S(\lambda_1,\lambda_2, \ldots,\lambda_d\big) \\
    &= S \circ \phi_1 \circ \tilde{\phi}\Big (
    \int  \psi_{\sigma_1} \big(\frac{1}{2} v^t \nabla^2L(x) v \big) d\mu(v),
    \int  \psi_{\sigma_2} \big( \frac{1}{2}v^t \nabla^2L(x) v \big) d\mu(v),
    \ldots,
    \int  \psi_{\sigma_d} \big( \frac{1}{2}v^t \nabla^2L(x) v \big) d\mu(v)
    \Big). 
\end{align*}
Now to complete the proof, we define a continuous function $\psi:\R \to \R^d$ as $\psi = \big(\psi_{\sigma_1}, \psi_{\sigma_2}, \ldots, \psi_{\sigma_d}\big)^t$, and a continuous function $\phi: \R^d \to \R$ as $ \phi= S \circ \phi_1 \circ \tilde{\phi}$, and observe that $S(\lambda_1,\lambda_2\ldots,\lambda_d) = S(x;\phi,\psi,\mu)$ for any $x \in \calA$. This completes the proof.
  \end{proof}

\section{Proof of Theorem \ref{theoremuniveigenfunction}}\label{proof_theoremuniveigenfunction}

\theoremuniveigenfunction*

\begin{proof}
We explicitly construct a set of functions/probability measures to achieve the desired representation. Indeed, let's take $m = d(d+1)/2$ and consider the following Dirac measures:  $\mu_i = \delta_{e_i}, i \in [d]$, and also $\mu_{ij} = \delta_{e_i+e_j}$, for any $i,j \in [d]$ such that $i < j$. Here,  $e_i$ denotes the unit vector in the $i$th coordinate in $\R^d$.  Now, note that we have 
\begin{align}
   (\nabla^2L(x))_{i,i} =  \int    v^t \nabla^2L(x) v  d\mu_i(v) 
\end{align}
for any $i \in [d]$, and 
\begin{align}
   2(\nabla^2L(x))_{i,j} + (\nabla^2L(x))_{i,i} + (\nabla^2L(x))_{j,j} =  \int    v^t \nabla^2L(x) v  d\mu_{i,j}(v),
\end{align}
for any $i,j \in [d]$ such that $i < j$.  The above system of linear equations has clearly a unique solution, as the Hessian matrix is symmetric. This means that, similar to the proof of Theorem \ref{theoremuniv}, one can find continuous functions $\psi: \R \to \R^m$ and $\phi: \R^m \to R$, along with $m$ constructed probability measures such that $S(\nabla^2 L(x)) = S(x;\phi,\psi,\mu)$ for any $x \in \R^d$, where  $\mu = \mu_1\otimes \mu_2 \otimes \ldots \otimes \mu_m$ is a product probability measure.
\end{proof}

\section{Proof of Theorem \ref{theorembias}}\label{proof_theorembias}

\theorembias*

\begin{proof}

For simplicity, we assume that $m=1$. The general proof for $m>1$ follows with a similar argument to this special case. Define an open set $U \subseteq \R^d$ as follows: 
\begin{align}
    U: = \big\{ x \in \R^d: \|\nabla L(x)\|_2 < \rho^2 \big\}. 
\end{align}
 Note that this set contain the zero-loss manifold, i.e., $\Gamma \subseteq U$. We study the behavior of the loss function on this open set.

Let us denote the sharpness term in the loss function $L_{(\phi,\psi,\mu)}(x)$ by $R_{\rho}(x)$:
\begin{align}
 L_{(\phi,\psi,\mu)}(x)&= L(x) + \rho^2 R_{\rho}(x) 
 \\& = L(x) 
    +   \rho^2 \phi \Big( \int  \psi \Big( 
    \frac{1}{\rho^2} \big( L(x+\rho v)- L(x) \big) 
    \Big) d\mu(v)\Big).
\end{align}
We first study the convergence of $R_{\rho}(x)$ to the corresponding sharpness measure  $S(x;\phi,\psi,\mu)$. Fix any point $x \in U$ and  note that using Taylor's  theorem for the function $L:\R^d \to \R$ and for any $v \in \R^d$, one has
\begin{align}
    \frac{1}{\rho^2} \big( L(x+\rho v)- L(x) \big)  &= \frac{1}{\rho^2} \Big( \rho \langle  \nabla L(x), v \rangle  + \frac{1}{2} \rho^2 v^t \nabla^2 L(x) v + O_{x}(\rho^3 \|v\|_2^3 \times \|v\|_2^{-1})\Big)\\
    & =   \rho^{-1} \langle  \nabla L(x), v \rangle  +    \frac{1}{2} v^t \nabla^2 L(x) v + O_{x}(\rho \|v\|_2^2 \times \|v\|_2^{-1}),
\end{align}
where in above we used the fact that $L(x)$ is third-order continuously differentiable and its third-order derivative satisfies the estimate
\begin{align}
    \max_{i,j,k \in \{1,2,3\}} |\partial_i \partial_j \partial_k L(v)|  = O(\| v\|^{-1}),
    \end{align}
    for $v \in \R^d$ as $\| v\|_2 \to \infty$. 

Note that using the assumption $x \in U$, we have that 
\begin{align}
    |\rho^{-1} \langle  \nabla L(x), v \rangle| \le  \rho^{-1} \|\nabla L(x)\|_2 \|v\|_2 <\rho \|v\|_2. 
\end{align}
Thus, we have 
\begin{align}
    \frac{1}{\rho^2} \big( L(x+\rho v)- L(x) \big)  =   \frac{1}{2}v^t \nabla^2 L(x) v + O_{x}(\rho (\|v\|_2^2 + \|v\|_2)). 
\end{align}
Note that we study the above approximation only for $x \in U$, and for small enough $\rho$, we know that $U$ is a precompact set. Therefore, we drop the dependence on $x$ in the error term above.

Now using the above approximation, we have 
\begin{align}
     \int  \psi \Big( 
    \frac{1}{\rho^2} \big( L(x+\rho v)- L(x) \big) 
    \Big) d\mu(v) = \int  \psi \Big( 
       \frac{1}{2}v^t \nabla^2 L(x) v + O(\rho (\|v\|_2^2 + \|v\|_2))
    \Big) d\mu(v).
\end{align}
Let us use Taylor's theorem for the function $\psi$ and write
\begin{align}
    \psi \Big( 
       \frac{1}{2}v^t \nabla^2 L(x) v + O(\rho (\|v\|_2^2 + \|v\|_2))
    \Big) = \psi\big(\frac{1}{2}v^t \nabla^2 L(x) v\big) + \rho \times O(\tilde{\psi}(v) (\|v\|_2^2 + \|v\|_2)),
\end{align}
where 
\begin{align}
    \tilde{\psi}(v) := \max_{|t|\le \|v\|_2} |\psi'(Ct^2)|,
\end{align}
and $C$ is a constant, and it's big enough to absorb the quadratic growth of $\nabla^2L(x)$; i.e., 
\begin{align}
    C > \max_{x \in \mathcal{X}}\max_{i \in [d]} |\lambda_{i}(\nabla^2 L(x))|. 
\end{align}
Therefore, we conclude that
\begin{align*}
     \int  \psi \Big( 
    \frac{1}{\rho^2} \big( L(x+\rho v)- L(x) \big) 
    \Big) d\mu(v) &= \int  \psi \big( \frac{1}{2}
       v^t \nabla^2 L(x) v \big ) d\mu(v)+ \rho \times  O\Big(\int \tilde{\psi}(v) (\|v\|_2^2 + \|v\|_2)
    d\mu(v) \Big)\\
    & = \int  \psi \big( 
       \frac{1}{2} v^t \nabla^2 L(x) v \big ) d\mu(v)+ O(\rho ),
\end{align*}
by the assumption. This allows us to conclude that
\begin{align}
    R_{\rho}(x) &= \phi \Big( \int  \psi \Big( 
    \frac{1}{\rho^2} \big( L(x+\rho v)- L(x) \big) 
    \Big) d\mu(v)\Big)\\
    & = \phi \Big( \int  \psi \big( 
       \frac{1}{2} v^t \nabla^2 L(x) v \big )  d\mu(v)\Big) + O(\rho)\\
    & = S(x;\phi,\psi,\mu) + O(\rho),
\end{align}
again, by assuming that 
\begin{align}
    \max_{x \in \mathcal{X}}\phi' \Big( \int  \psi \big( 
       \frac{1}{2} v^t \nabla^2 L(x) v \big ) d\mu(v)\big) < \infty,  
\end{align}
which holds by the compactness of $\mathcal{X}$, and also using $U \subseteq \mathcal{X}$.

Now according to the assumption, for some $u \in U$, we have
\begin{align}
    L(u) + \rho^2 R_{\rho}(u) - \inf_{x \in U}\Big (
    L(x) + \rho^2 R_{\rho}(x)
    \Big) \le \Delta \rho^2, 
\end{align}
for some  optimally gap $\Delta$. Using the following proven approximation
\begin{align}
    R_{\rho}(x) = S(x;\phi,\psi,\mu) + O(\rho), 
\end{align}
we conclude that
\begin{align}
    L(u) + \rho^2 S(u;\phi,\psi,\mu) - \inf_{x \in U}\Big (
    L(x) + \rho^2 S(x;\phi,\psi,\mu)
    \Big) \le (\Delta + O(\rho)) \rho^2. 
\end{align}

Now by proof by contradiction, assume that 
\begin{align}
    L(u) \ge \inf_{x \in U} L(x) + (\Delta + \delta) \rho^2,
\end{align}
for some strictly positive $\delta$, as $\rho \to 0^+$. Note that $ \inf_{x \in U} L(x) = 0$ as $\Gamma \subseteq U$. Thus, we can conclude that
\begin{align}
    \rho^2 S(u;\phi,\psi,\mu) +(\Delta + \delta) \rho^2 &\le L(u) + \rho^2 S(u;\phi,\psi,\mu) \\
    &\le \inf_{x \in U}\Big (
    L(x) + \rho^2 S(x;\phi,\psi,\mu)
    \Big)+ (\Delta + O(\rho)) \rho^2\\
    &\le \rho^2 \inf_{x \in \Gamma}S(x;\phi,\psi,\mu) + (\Delta + O(\rho)) \rho^2,
\end{align}
since $\Gamma \subseteq U$. This shows that
\begin{align}
     S(u;\phi,\psi,\mu) \le  \inf_{x \in \Gamma}S(x;\phi,\psi,\mu) - \delta +O(\rho). 
\end{align}
This must hold for as $\rho \to 0^+$. However, as $\rho \to 0^+$, we have that $U_{\rho} \to \Gamma$. This means that 
\begin{align}
     S(u;\phi,\psi,\mu) \le  \inf_{x \in \Gamma}S(x;\phi,\psi,\mu) - \delta,
\end{align}
for some $u \in \Gamma$, which is a contradiction. This shows that
\begin{align}
    L(u) \le \inf_{x \in U} L(x) + (\Delta + o(1)) \rho^2.
\end{align}

Also, to prove the next part of the theorem, for any $u \in U$ satisfying the assumptions, similarly we can show 
\begin{align}
    \rho^2 S(u;\phi,\psi,\mu) &\le L(u) + \rho^2 S(u;\phi,\psi,\mu)  \\
    & = \inf_{x \in U}\Big (
    L(x) + \rho^2 S(x;\phi,\psi,\mu)
    \Big) +(\Delta + O(\rho)) \rho^2\\
    & \le \rho^2 \inf_{x \in \Gamma} S(x;\phi,\psi,\mu) +(\Delta + O(\rho))\rho^2,
\end{align}
which implies that
\begin{align}
    S(u;\phi,\psi,\mu) \le \inf_{x \in \Gamma} S(x;\phi,\psi,\mu) + (\Delta + O(\rho)).
\end{align}
The proof is thus complete.

\end{proof}

\section{Proof of Theorem \ref{theoremscale}}\label{proof_theoremscale}

\theoremscale*

\begin{proof}

Let $D \in \mathbb{R}^{d \times d}$ be an arbitrary diagonal matrix. Then, 
\begin{align}
    S(Dx;\phi,\psi,\mu)=\phi \Big( \int  \psi \big(\frac{1}{2} v^t \nabla^2L(x)\Big|_{Dx} v \big) d\mu(v)\Big).
\end{align}
But note that by the assumption
\begin{align}
    L(x) = L(Dx) \implies \nabla^2L(x) = D^t \nabla^2L(x)\Big |_{Dx} D.
\end{align}
Therefore, 
\begin{align}
    S(Dx;\phi,\psi,\mu)=\phi \Big( \int  \psi \big( \frac{1}{2}v^t D^{-1}\nabla^2L(x) D^{-1}v \big) d\mu(v)\Big).
\end{align}
Now define a new variable $u:=D^{-1}v$. Then,  
\begin{align}
    d\mu(v) &= f(\prod_{i=1}^dv_i)\prod_{i=1}^d dv_i = f(\prod_{i=1}^d D_{i,i} \prod_{i=1}^d u_i)\prod_{i=1}^d D_{i,i} \prod_{i=1}^d du_i\\
    & = f(\det(D) \prod_{i=1}^d u_i)\det(D) \prod_{i=1}^d du_i \\
    & = f(\prod_{i=1}^du_i)\prod_{i=1}^d du_i. 
\end{align}
Therefore, we conclude that 
\begin{align}
    S(Dx;\phi,\psi,\mu)=\phi \Big( \int  \psi \big( \frac{1}{2}u^t \nabla^2L(x) u \big) d\mu(u)\Big) = S(x;\phi,\psi,\mu),
\end{align}
and this completes the proof.

\end{proof}

\begin{lemma}\label{lemma:scale_invatiant_measures}
    For any scale-invariant measure $\mu$ that is absolutely continuous with respect to the Lebesgue measure on $\R^d$, one has
    \[d \mu(x) = f\big(\prod_{i = 1}^{d} x_i\big) \prod_{i = 1}^{d} d x_i.
    \]
\end{lemma}
\begin{proof}
Any measure $\mu$ that is absolutely continuous with respect to the Lebesgue measure  on $\R^d$ can be written as
\[ 
d \mu(x_1, x_2, \dots, x_d) = \tilde{f} (x_1, x_2, \dots, x_d) \prod_{i = 1}^{d} dx_i. 
\]
For any non-zero choice of $x_1, x_2, \dots, x_d$, by scale invariance property of $d \mu$,
\begin{align*} 
& d \mu(x_1, x_2, \dots, x_d) = d \mu(1, 1, \dots, 1, \prod_{i = 1}^d x_i) \\ \Rightarrow & \tilde{f} (x_1, x_2, \dots, x_d) \prod_{i = 1}^{d} dx_i = \tilde{f} (1, 1, \dots, 1, \prod_{i = 1}^d x_i) \prod_{i = 1}^{d} dx_i \\ \Rightarrow & \tilde{f} (x_1, x_2, \dots, x_d) = \tilde{f} (1, 1, \dots, 1, \prod_{i = 1}^d x_i). 
\end{align*} 

Now, it suffices to choose $f(\prod_{i = 1}^{d} x_i) := \tilde{f} (1, 1, \dots, 1, \prod_{i = 1}^d x_i)$ almost surely and the proof is concluded.
\end{proof}

\section{Proof of Theorem \ref{theoremgroup}}\label{proof_theoremgroup}

\theoremgroup*

\begin{proof}

We start by evaluating $S(A_g x;\phi,\psi,\mu)$.
\begin{align}
    S(Dx;\phi,\psi,\mu)=\phi \Big( \int  \psi \big(\frac{1}{2} v^t \nabla^2L(x)\Big|_{A_g x} v \big) d\mu(v)\Big).
\end{align}
But again here, note that by the assumption
\begin{align}
    L(x) = L(A_g x) \implies \nabla^2L(x) = A_g ^t \nabla^2L(x)\Big |_{A_g x} A_g.
\end{align}
Therefore, 
\begin{align}
    S(A_g x;\phi,\psi,\mu)=\phi \Big( \int  \psi \big( \frac{1}{2}v^t A_g^{-1}\nabla^2L(x) A_g^{-1}v \big) d\mu(v)\Big).
\end{align}
Now define a new variable $u:=A_g^{-1}v$. Therefore, we conclude that 
\begin{align}
    S(A_g x;\phi,\psi,\mu)=\phi \Big( \int  \psi \big( \frac{1}{2}u^t \nabla^2L(x) u \big) d\mu(u)\Big) = S(x;\phi,\psi,\mu),
\end{align}
and this completes the proof.

\end{proof}

\section{$(\phi, \psi, \mu)$-Sharpness-Aware Minimization Algorithm}

To propose an algorithm for the general case (i.e., arbitrary $m$), we compute the gradient of 
\begin{align}
        \tilde{R}_{\rho}(x) \coloneqq \phi \Big(&  
        \frac{1}{n} \sum_{i = 1}^{n}  \psi_1 \Big( \frac{1}{\rho^2} \big( L(x+\rho v_{i,1})- L(x) \big) \Big),\\
        &\frac{1}{n} \sum_{i = 1}^{n}  \psi_2 \Big( \frac{1}{\rho^2} \big( L(x+\rho v_{i,2})- L(x) \big) \Big),\\
&\hspace{2cm}\ldots\\
& \frac{1}{n} \sum_{i = 1}^{n}  \psi_m \Big( \frac{1}{\rho^2} \big( L(x+\rho v_{i,m})- L(x) \big) \Big) \Big ),
\end{align}
where $v_{i,\ell} \overset{\text{i.i.d.}}{\sim} \mu_{\ell}$ for each $\ell \in [m]$.  Note that $\psi = (\psi_1,\psi_2,\ldots,\psi_m)^t$ for some scalar functions $\psi_{\ell}$, $\ell \in [m]$. Let $\partial_{\ell}\phi$ denote partial derivatives of the function $\phi: \R^m \to \R$, for any $\ell \in [m]$. Then,  
\begin{align}
    \rho^2 \nabla \tilde{R}_{\rho}(x) & = \sum_{\ell = 1}^m \partial_{\ell}\phi\Big(\sum_{i=1}^n \frac{1}{n}\psi_{\ell} \Big(\frac{1}{\rho^2} \big(L(x_t+\rho v_{i,\ell}) - L(x_t)\big)\Big) \Big) \\
      & \hspace{1cm}\times  \sum_{i=1}^n \frac{1}{n} \Big \{
    \psi_\ell'\Big(\frac{1}{\rho^2} \big(L(x_t + \rho v_{i,\ell}) - L(x_t)\big) \Big) 
    \times \Big(\nabla L(x_t + \rho v_{i,\ell}) - \nabla L(x_t)\Big) \Big \},
\end{align}
and this leads to \cref{algo-gen}.

\begin{algorithm*}[!t]
   \caption{  $(\phi, \psi, \mu)$-Sharpness-Aware Minimization Algorithm (with arbitrary $m$)}
\begin{algorithmic}\label{algo-gen}
   \STATE {\bfseries Input:} The triplet $(\phi, \psi, \mu)$, training loss $L(x)$, step size $\eta$, perturbation parameter $\rho$, number of samples $n$, 
   \STATE {\bfseries Output:} Model parameters $x_t$ trained with $(\phi, \psi, \mu)$-sharpness-aware minimization algorithm
   \STATE Initialization: $x \gets x_0$ and  $t \gets 0$
   \WHILE{1}
   \STATE Sample $v_{i,\ell} \overset{\text{i.i.d.}}{\sim} \mu_{\ell}$, for any $i \in [n]$ and $\ell \in [m]$
    \STATE Compute the following: \vspace{-0.5cm}\begin{align*}
       g_t &= \nabla L(x_t) +  \sum_{\ell = 1}^m \partial_{\ell}\phi\Big(\sum_{i=1}^n \frac{1}{n}\psi_{\ell} \Big(\frac{1}{\rho^2} \big(L(x_t+\rho v_{i,\ell}) - L(x_t)\big)\Big) \Big) \\
      & \hspace{2cm} \times  \sum_{i=1}^n \frac{1}{n} \Big \{
    \psi_\ell'\Big(\frac{1}{\rho^2} \big(L(x_t + \rho v_{i,\ell}) - L(x_t)\big) \Big) 
    \times \Big(\nabla L(x_t + \rho v_{i,\ell}) - \nabla L(x_t)\Big) \Big \}.
    \end{align*} \vspace{-0.5cm}
    \STATE Update the parameters: $x_{t+1} = x_t - \eta g_t$
    \STATE $t \gets t+1$
   \ENDWHILE
   \end{algorithmic}
\end{algorithm*}

\section{Experiments}
\label{sec:full_experiments}

\subsection{Experimental Details}
We now describe experimental details that were omitted from the main text.

For CIFAR10 and CIFAR100, we apply random crops and random horizontal flips. We use a momentum term of 0.9 for all datasets and a weight decay of 5e-4 for CIFAR10 and SVHN and 1e-3 for CIFAR100. We use batch size 128 and train for 200 epochs. We use a multi-step schedule where the learning rate is initially 0.1 and decays by a multiplicative factor of 0.1 every 50 epochs. We run each experiment with four different random seeds to assess statistical significance. We use 1280 training examples and 100 noise samples to estimate the Frobenius norm and trace via Hessian-vector products.
We set $\rho$ to 1.0 for Det-SAM, 0.01 for Trace-SAM, and sweep it in \{0.005, 0.01\} for Frob-SAM. For Det-SAM and Trace-SAM we sweep $\lambda$ in \{0.01, 0.1, 1.0\} and set $n=1$. For Frob-SAM, we sweep $\lambda$ in \{0.0001, 0.001, 0.005, 0.01, 0.05, 0.1\} and set $n=2$. The hyper-parameters selected for each setting is given in \cref{table:hparams_lam} and \cref{table:hparams_rho}.

We use the PyHessian library~\citep{yao2020pyhessian} to estimate the trace of the Hessian. Adapting this library, we estimate the Frobenius norm (squared) as $\frac{1}{k}\sum_{i=1}^k \|H z_i\|_2^2$, where $H$ is the Hessian matrix and $z_i \sim \mathcal{N}(0, I_d)$.

\begin{table}[!h]
\centering
\small
\setlength{\tabcolsep}{2pt}
\begin{tabular}{r||ccc|ccc|ccc}
\multicolumn{1}{l||}{} & CIFAR10 & CIFAR100 & SVHN  & CIFAR10-S & CIFAR100-S & SVHN-S & CIFAR10-C & CIFAR100-C & SVHN-C \\ \hline \hline
Frob-SAM              & 5e-3   & 1e-4    & 5e-3 & 1e-4     & 1e-4      & 1e-3  & 0.01     & 0.01      & 1e-3  \\
Trace-SAM             & 1       & 0.1      & 0.1   & 0.01      & 0.01       & 1      & 1         & 0.1        & 0.01   \\
Det-SAM               & 1       & 0.1      & 0.1   & 0.01      & 1          & 1      & 0.01      & 0.1        & 0.1    \\ \hline
\end{tabular}
\caption{Choice of regularizer weight $\lambda$ for different datasets.}
\label{table:hparams_lam}
\end{table}

\begin{table}[!h]
\centering
\small
\setlength{\tabcolsep}{2pt}
\begin{tabular}{l||ccc|ccc|ccc}
                                       & CIFAR10 & CIFAR100 & SVHN & CIFAR10-S & CIFAR100-S & SVHN-S & CIFAR10-C & CIFAR100-C & SVHN-C \\ \hline \hline
\multicolumn{1}{r||}{Frob-SAM} & 5e-3    & 5e-3     & 5e-3 & 1e-2      & 5e-3       & 5e-3   & 1e-2      & 5e-3       & 1e-2   \\ \hline
\end{tabular}
\caption{Choice of $\rho$ for Frob-SAM for different datasets.}
\label{table:hparams_rho}
\end{table}

\subsection{Additional Plots}
All training plots are shown in \cref{fig:accuracy_plots}.

\begin{figure*}[!h]
    \centering
    \includegraphics[width=0.3\textwidth]{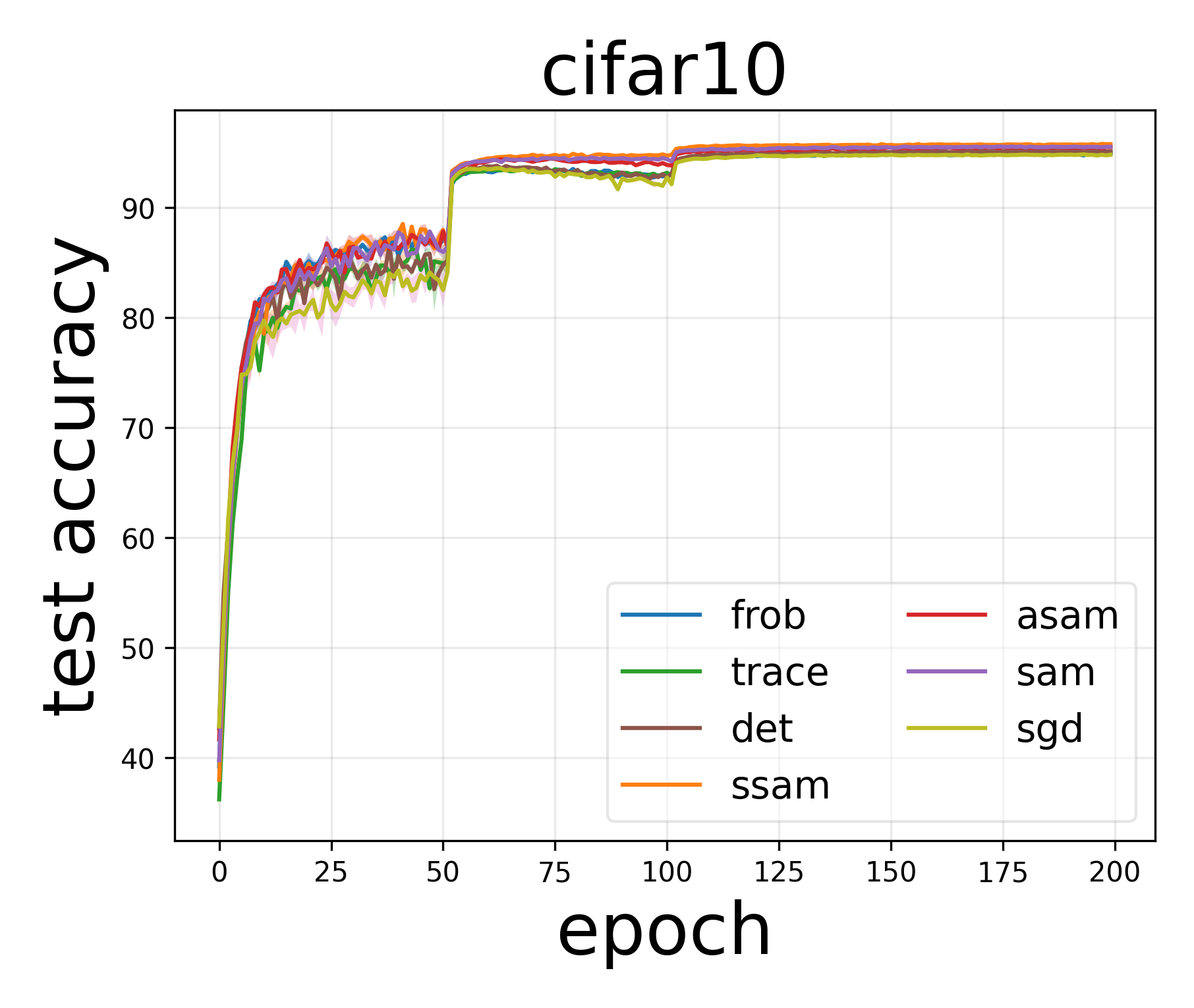}
    \includegraphics[width=0.3\textwidth]{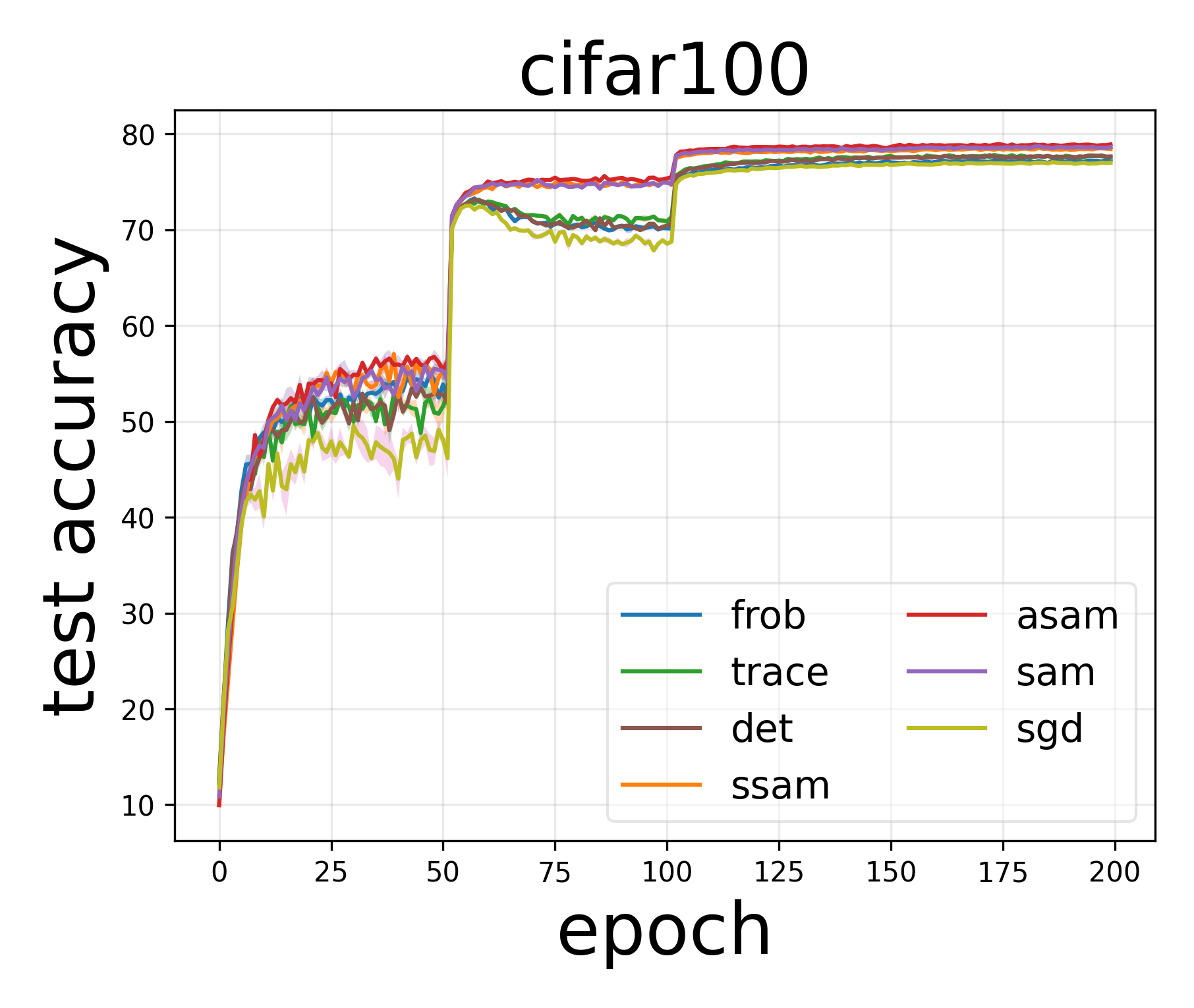}
    \includegraphics[width=0.3\textwidth]{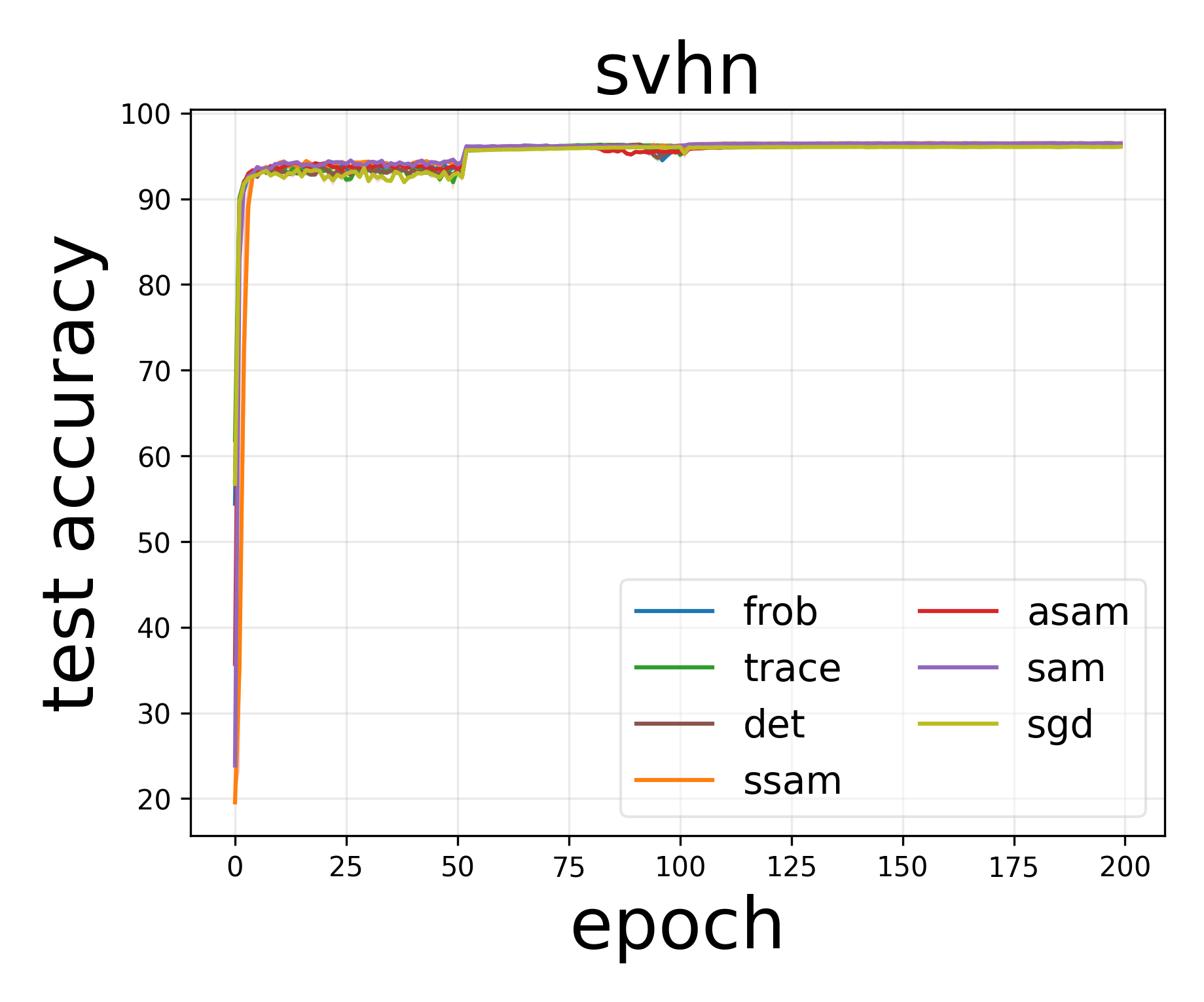}
    \\
    \includegraphics[width=0.3\textwidth]{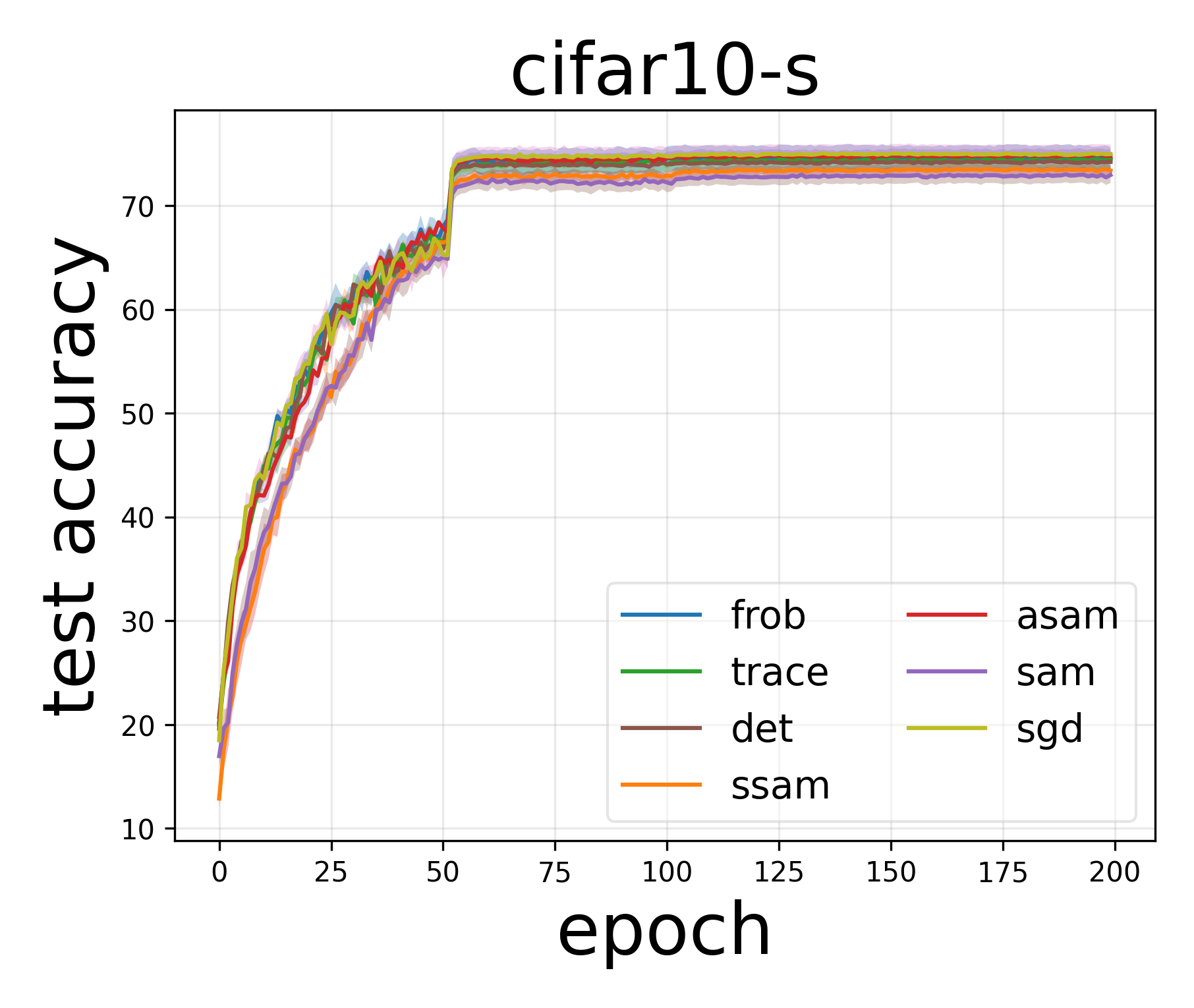}
    \includegraphics[width=0.3\textwidth]{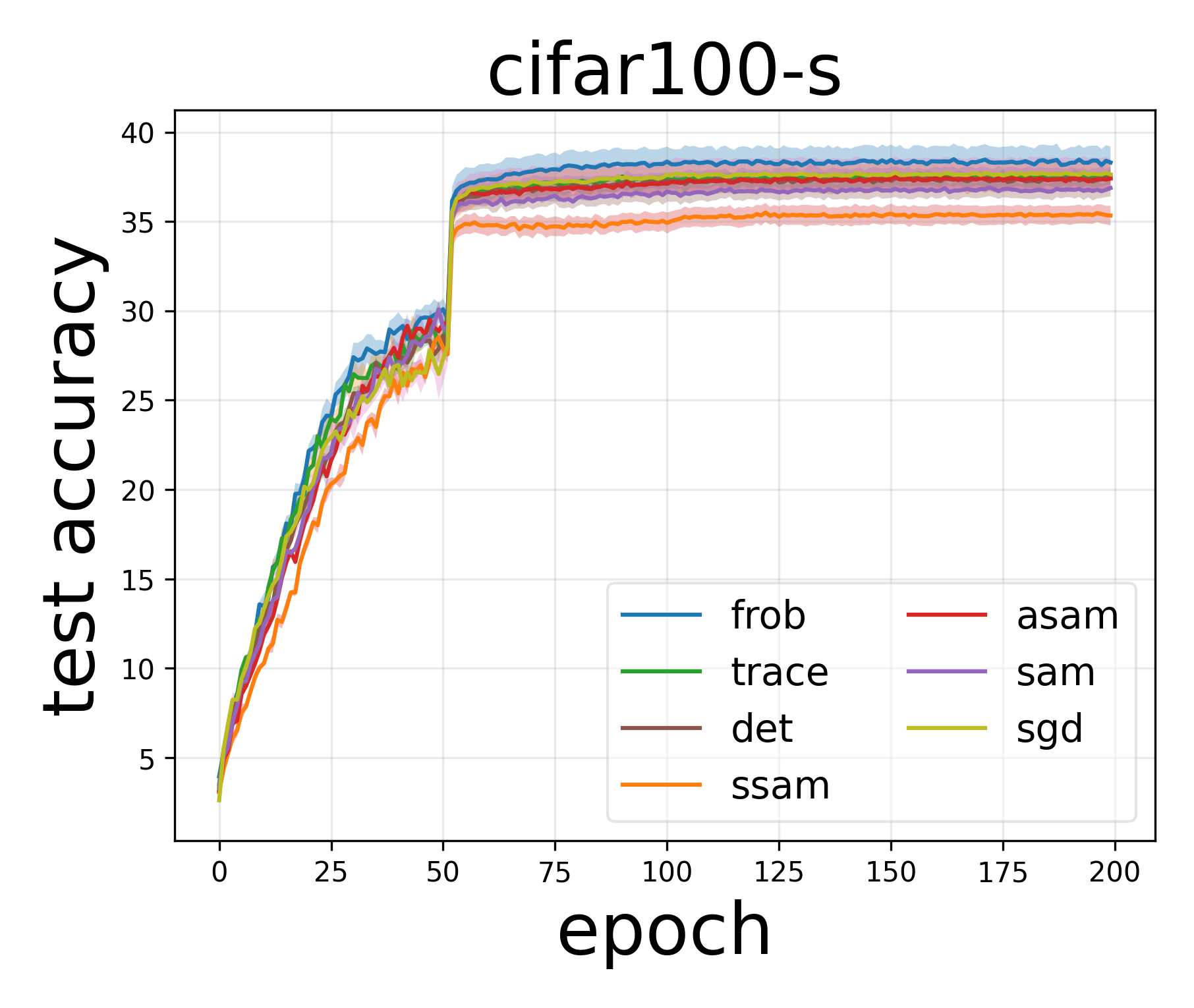}
    \includegraphics[width=0.3\textwidth]{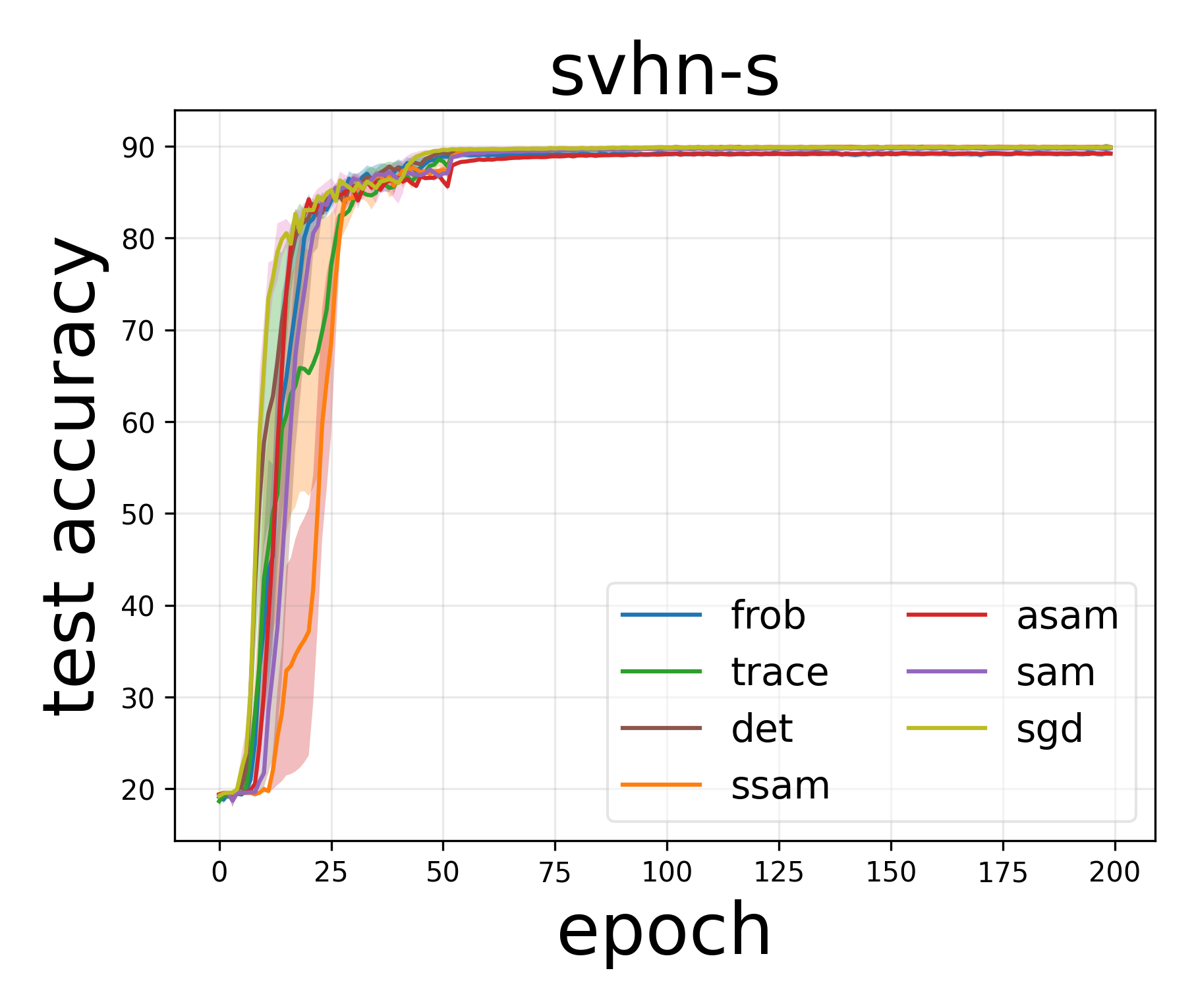}
    \\
    \includegraphics[width=0.3\textwidth]{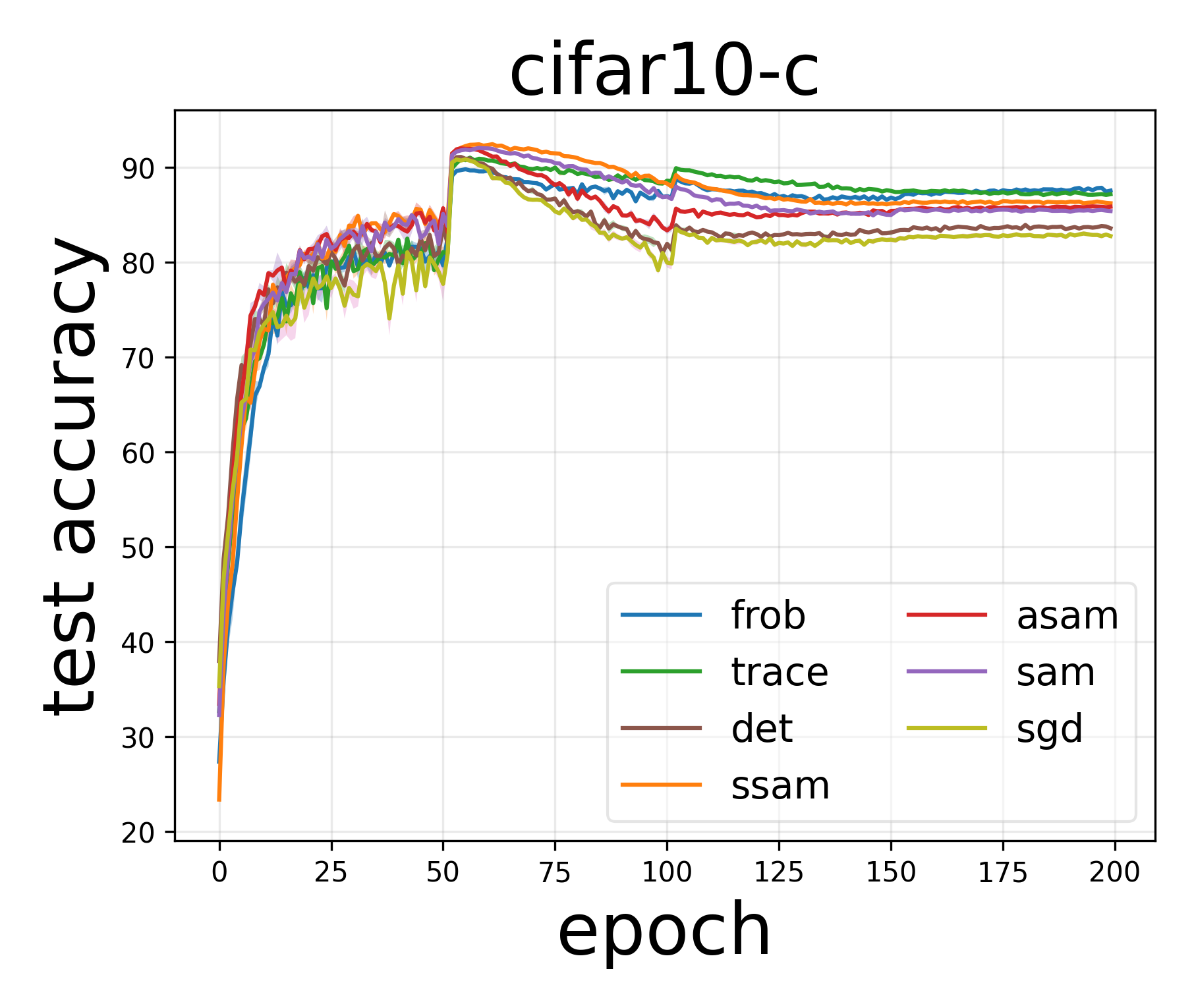}
    \includegraphics[width=0.3\textwidth]{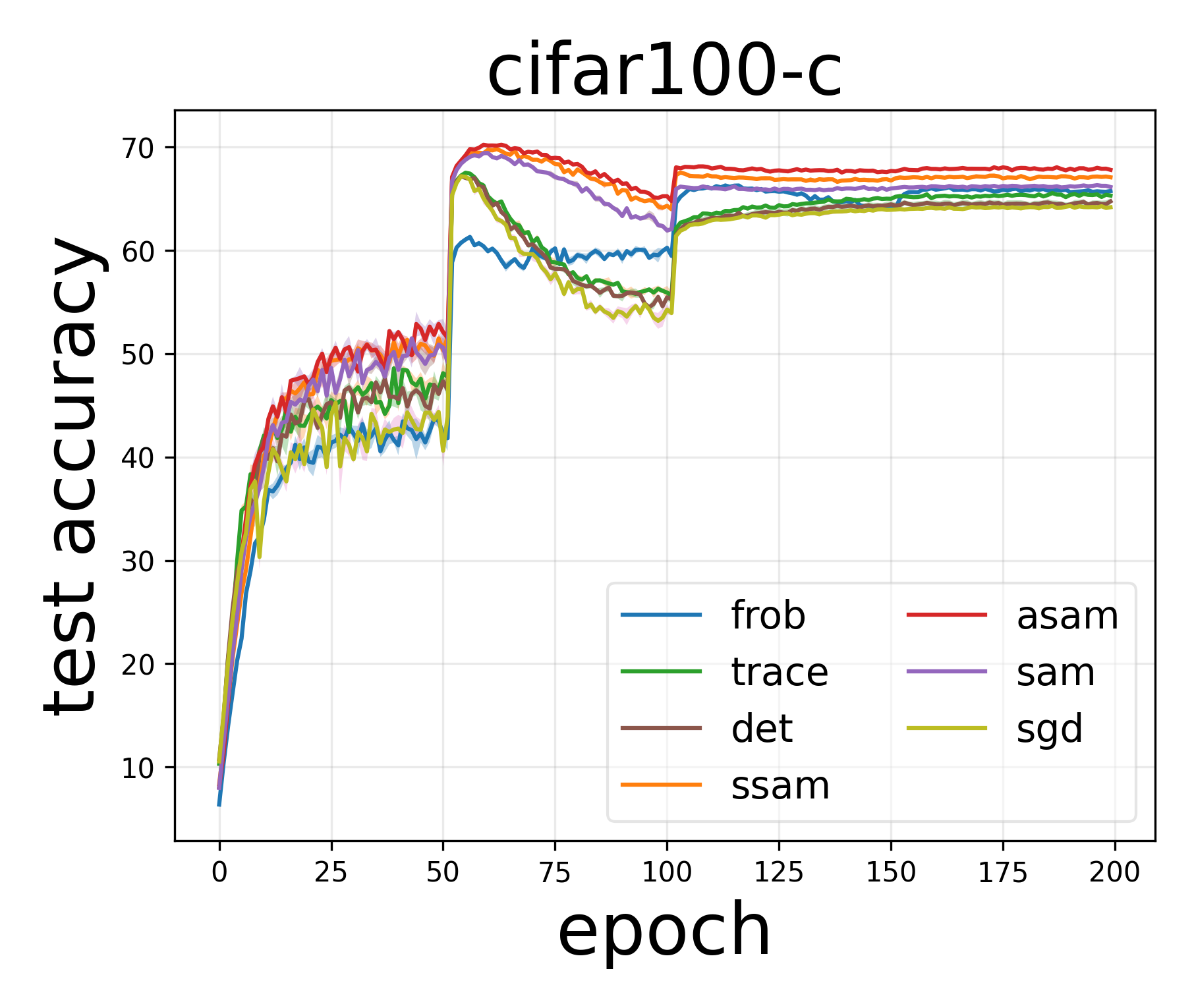}
    \includegraphics[width=0.3\textwidth]{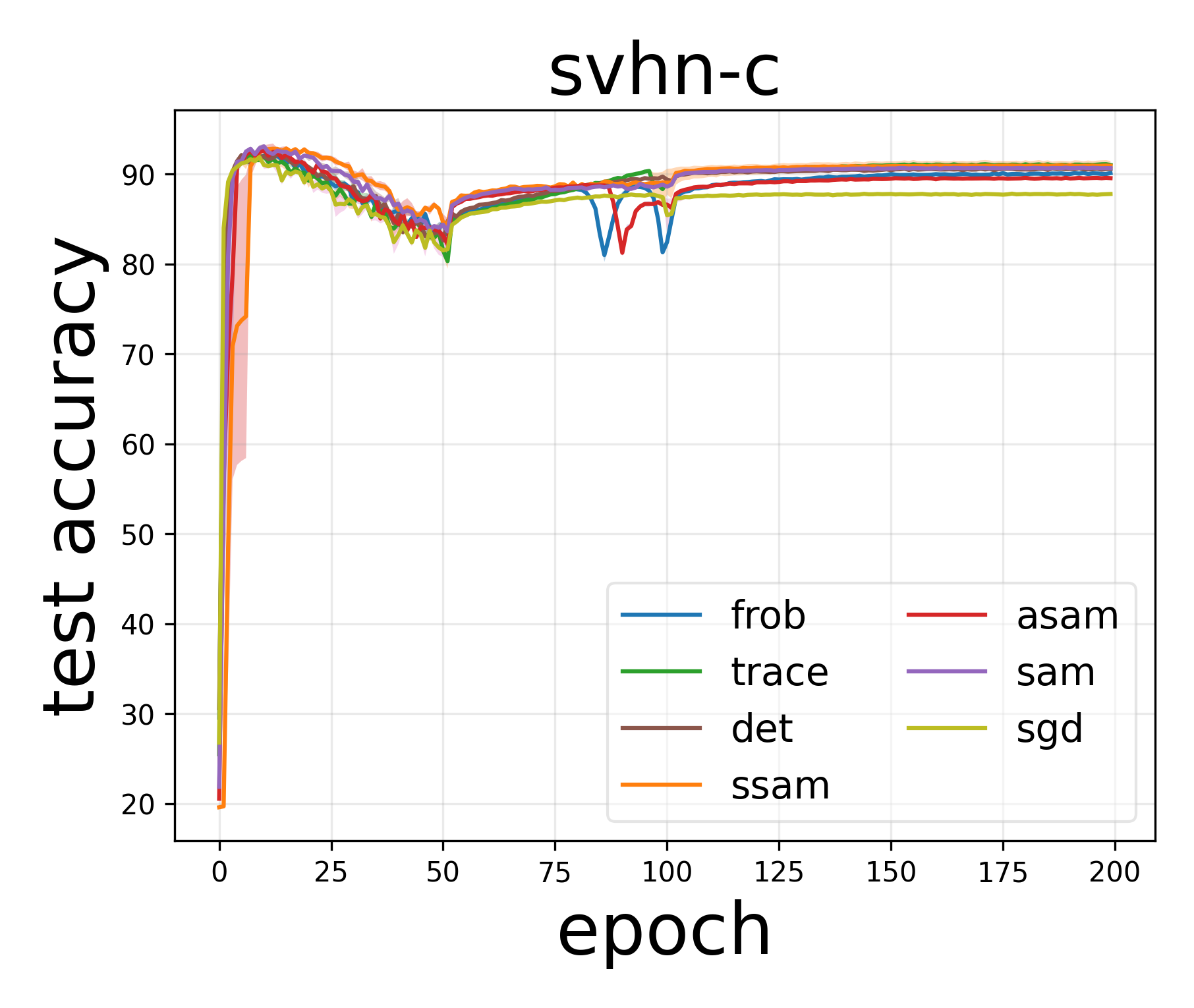}
    \caption{Training plots for all settings. One standard error is shaded.}
    \label{fig:accuracy_plots}
\end{figure*}

\end{document}